\PassOptionsToPackage{authoryear}{natbib}
\documentclass[english]{article}

\usepackage{natbib}
\usepackage{lmodern}
\usepackage{fullpage}
\usepackage[utf8]{inputenc}
\usepackage[T1]{fontenc}
\usepackage{hyperref}
\usepackage{url}
\usepackage{booktabs}
\usepackage{amsfonts}
\usepackage{nicefrac}
\usepackage{microtype}

\usepackage{float}
\usepackage{amsmath}
\usepackage{amsthm}
\usepackage{amssymb}
\usepackage{thm-restate}
\usepackage{algorithm,algpseudocode}
\usepackage{titlesec}
\usepackage{babel}

\makeatletter

\floatstyle{ruled}
\newfloat{algorithm}{tbp}{loa}
\providecommand{\algorithmname}{Algorithm}
\floatname{algorithm}{\protect\algorithmname}

\theoremstyle{plain}
\newtheorem{thm}{\protect\theoremname}
\theoremstyle{plain}

\theoremstyle{definition}
\newtheorem{defn}[thm]{\protect\definitionname}
\theoremstyle{plain}

\hypersetup{colorlinks=true, urlcolor=cyan, citecolor=blue}
\DeclareMathOperator*{\argmax}{arg\,max}
\DeclareMathOperator*{\argmin}{arg\,min}

\algnewcommand{\Parameters}[1]{\item[\textbf{Parameters:}]{#1}}
\algnewcommand{\Init}[1]{\item[\textbf{Initialize:}]{#1}}

\makeatother

\providecommand{\corollaryname}{Corollary}
\providecommand{\definitionname}{Definition}
\providecommand{\lemmaname}{Lemma}
\providecommand{\theoremname}{Theorem}

\begin{document}
\title{\huge Individual Regret in Cooperative Nonstochastic
Multi-Armed Bandits}
\author{\textbf{Yogev Bar-On}\\
Tel Aviv University\\
\texttt{baronyogev@gmail.com}\and\textbf{Yishay Mansour}\\
Tel Aviv University\\
and Google Research\\
\texttt{mansour.yishay@gmail.com}}
\date{}
\maketitle
\begin{abstract}
We study agents communicating over an underlying network by exchanging
messages, in order to optimize their individual regret in a common
nonstochastic multi-armed bandit problem. We derive regret minimization
algorithms that guarantee for each agent $v$ an individual expected
regret of $\widetilde{O}\left(\sqrt{\left(1+\frac{K}{\left|\mathcal{N}\left(v\right)\right|}\right)T}\right)$,
where $T$ is the number of time steps, $K$ is the number of actions
and $\mathcal{N}\left(v\right)$ is the set of neighbors of agent
$v$ in the communication graph. We present algorithms both for the
case that the communication graph is known to all the agents, and
for the case that the graph is unknown. When the graph is unknown,
each agent knows only the set of its neighbors and an upper bound
on the total number of agents. The individual regret between the models
differs only by a logarithmic factor. Our work resolves an open problem
from \citep{cesa2019delay}.
\end{abstract}

\section{\label{sec:introduction}Introduction}

The multi-armed bandit (MAB) problem is one of the most basic models
for decision making under uncertainty. It highlights the agent's uncertainty
regarding the losses it suffers from selecting various actions. The
agent selects actions in an online fashion - each time step the agent
selects a single action and suffers a loss corresponding to that action.
The agent's goal is to minimize its cumulative loss over a fixed horizon
of time steps. The agent observes only the loss of the action it selected
each step. Therefore, the MAB problem captures well the crucial trade-off
between exploration and exploitation, where the agent needs to explore
various actions in order to gather information about them.

MAB research discusses two main settings: the stochastic setting,
where the losses of each action are sampled i.i.d. from an unknown
distribution, and the nonstochastic (adversarial) setting, where we
make no assumptions about the loss sequences. In this work we consider
the nonstochastic setting and the objective of minimizing the regret
- the difference between the agent's cumulative loss and the cumulative
loss of the best action in hindsight. It is known that a regret of
the order of $\Theta\left(\sqrt{KT}\right)$ is the best that can
be guaranteed, where $K$ is the number of actions and $T$ is the
time horizon. In contrast, when the losses of all actions are observed
(full-information feedback) the regret can be of the order of $\Theta\left(\sqrt{T\ln K}\right)$
(see, e.g., \citep{cesa2006prediction,bubeck2012regret}).

The main focus of our work is to consider agents that are connected
in a communication graph, and can exchange messages in each step,
in order to reduce their individual regret. This is possible since
the losses depend only on the action and the time step, but not on
the agent.

One extreme case is when the communication graph is a clique, i.e.,
any pair of agents can communicate directly. In this case, the agents
can run the well known Exp3 algorithm \citep{auer2002nonstochastic},
and guarantee each a regret of $O\left(\sqrt{T\ln K}\right)$, assuming
there are at least $K$ agents (see \citep{seldin2014prediction,cesa2019delay}).
However, in many motivating applications, such as distributed learning,
or communication tasks such as routing, the communication graph is
not a clique.

The work of \citet{cesa2019delay} studies a general communication
graph, where the agents can communicate in order to reduce their regret.
The paper presents the Exp3-Coop algorithm, which achieves an expected
regret \textbf{when averaged over all agents} of $\widetilde{O}\left(\sqrt{\left(1+\frac{K}{N}\alpha\left(G\right)\right)T}\right)$,
where $\alpha\left(G\right)$ is the independence number of the communication
graph $G$, and $N$ is the number of agents. The question of whether
it is possible to obtain a low \textbf{individual regret}, that holds
simultaneously for all agents, was left as an open question. We answer
this question affirmatively in this work.

Our main contribution is an individual expected regret bound, which
holds for each agent $v$, of order
\[
\widetilde{O}\left(\sqrt{\left(1+\frac{K}{\left|\mathcal{N}\left(v\right)\right|}\right)T}\right),
\]
where $\mathcal{N}\left(v\right)$ is the set of neighbors of agent
$v$ in the communication graph. We remark that our result also implies
the previous average regret bound.

The main idea of our algorithm is to artificially partition the graph
into disjoint connected components. Each component has a center agent,
which is in some sense the leader of the component. The center agent
has (almost) the largest degree in the component, and it selects actions
using the Exp3-Coop algorithm. By observing the outcomes of its immediate
neighboring agents, the center agent can guarantee its own desired
individual regret. The main challenge is to create such components
with a relatively small diameter, so that the center will be able
to broadcast its information in a short time to all the agents in
the component. Special care is given to relate the agents' local parameters
(degree) to the global component parameters (degree of the center
agent and the broadcast time).

We consider both the case that the communication graph is known to
all the agents in advance (the informed setting), and the case that
the graph is unknown (the uninformed setting). In the uninformed setting,
we assume each agent knows its local neighborhood (i.e., the set of
its neighbors), and an upper bound on the total number of agents.
The regret bound in the uninformed setting is higher by a logarithmic
factor and the algorithm is more complex.

In the next section, we formally define our model, and review preliminary
material. Section \ref{sec:centerBased} shows the center-based policy,
given a graph partition. We then present our graph partitioning algorithms
in Section \ref{sec:partitioningTheGraph}. The analysis is given
in Section \ref{sec:regretAnalysis}. Our work is concluded in Section
\ref{sec:conclusions}.

\subsection{Additional related works}

The cooperative nonstochastic MAB setting was introduced by \citet{awerbuch2008competitive},
where they bound the average regret, when some agents might be dishonest
and the communication is done through a public channel (clique network).
The previously mentioned \citep{cesa2019delay}, also considers the
issue of delays, and presents a bound on the average regret for a
general graph of order $\widetilde{O}\left(\sqrt{\left(d+\frac{K}{N}\alpha\left(G\right)\right)T}+d\right)$,
when messages need $d$ steps to arrive. Dist-Hedge, introduced by
\citet{sahu2017dist}, considers a network of forecasting agents,
with delayed and inexact losses, and derives a sub-linear individual
regret bound, that also depends on spectral properties of the graph.
More recently, \citet{cesa2019cooperative} studied an online learning
model where only a subset of the agents play at each time step, and
showed matching upper and lower bounds on the average regret of order
$\sqrt{\alpha\left(G\right)T}$ when the set of agents that play each
step is chosen stochastically. When the set of agents is chosen arbitrarily,
the lower bound becomes $T$.

In the stochastic setting, \citet{landgren2016distributed,landgren2016distributed2}
presented a cooperative variant of the well-known UCB algorithm, that
uses a consensus algorithm for estimating the mean losses, to obtain
a low average regret. More cooperative variants of the UCB algorithm
that yield a low average regret were presented by \citet{kolla2018collaborative}.
They also showed a policy, where like in the methods in this work,
agents with a low degree follow the actions of agents with a high
degree. Stochastic MAB over P2P communication networks were studied
by \citet{szorenyi2013gossip}, which showed that the probability
to select a sub-optimal arm reduces linearly with the number of peers.
The case where only one agent can observe losses was investigated
by \citet{kar2011bandit}. This agent needs to broadcast information
through the network, and it was shown this is enough to obtain a low
average regret.

Another multi-agent research area involve agents that compete on shared
resources. The motivation comes from radio channel selection, where
multiple devices need to choose a radio channel, and two or more devices
that use the same channel simultaneously interfere with each other.
In this setting, many papers assume agents cannot communicate with
each other, and do not receive a reward upon collision - where more
than one agent tries to choose the same action at the same step. The
first to give regret bounds on this variant are \citet{avner2014concurrent},
that presented an average regret bound of order $O\left(T^{\frac{2}{3}}\right)$
in the stochastic setting. Also in the stochastic setting, \citet{rosenski2016multi}
showed an expected average regret bound of order $O\left(\frac{K}{\Delta^{2}}\ln\left(\frac{K}{\delta}\right)+N\right)$
that holds with probability $1-\delta$, where $\Delta$ is the minimal
gap between the mean rewards (notice that this bound is independent
of $T$). In the same paper, they also studied the case that the number
of agents may change each step, and presented a regret bound of $\widetilde{O}\left(\sqrt{xT}\right)$,
where $x$ is the total number of agents throughout the game. \citet{bistritz2018distributed}
consider the case that different agents have different mean rewards,
and each agent has a different unique action it should choose to maximize
the total regret. They showed an average regret of order $O\left(\log^{2+\epsilon}T\right)$
for every $\epsilon>0$, where the $O$-notation hides the dependency
on the mean rewards.More recently, \citet{alatur2019multi}
studied the adversarial setting, and proved an average regret bound
of $\widetilde{O}\left(N^{\frac{1}{3}}K^{\frac{2}{3}}T^{\frac{2}{3}}\right)$.

\section{\label{sec:preliminaries}Preliminaries}

We consider a nonstochastic multi-armed bandit problem over a finite
action set $A=\left\{ 1,\dots,K\right\} $ played by $N$ agents.
Let $G=\left\langle V,E\right\rangle $ be an undirected connected
communication graph for the set of agents $V=\left\{ 1,\dots,N\right\} $,
and denote by $\mathcal{N}\left(v\right)$ the neighborhood of $v\in V$,
including itself. Namely,
\[
\mathcal{N}\left(v\right)=\left\{ u\in V\mid\left\langle u,v\right\rangle \in E\right\} \cup\left\{ v\right\} .
\]
At each time step $t=1,2,\dots,T$, each agent $v\in V$ draws an
action $I_{t}\left(v\right)\in A$ from a distribution $\boldsymbol{p}_{t}^{v}=\left\langle p_{t}^{v}\left(1\right),\dots,p_{t}^{v}\left(K\right)\right\rangle $
on $A$. It then suffers a loss $\ell_{t}\left(I_{t}\left(v\right)\right)\in\left[0,1\right]$
which it observes. Notice the loss does not depend on the agent, but
only on the time step and the chosen action. Thus, agents that pick
the same action at the same step will suffer the same loss. We also
assume the adversary is oblivious, i.e., the losses do not depend
on the agents' realized actions. In the end of step $t$, each agent
sends a message
\[
m_{t}\left(v\right)=\left\langle v,t,I_{t}\left(v\right),\ell_{t}\left(I_{t}\left(v\right)\right),\boldsymbol{p}_{t}^{v}\right\rangle 
\]
to all the agents in its neighborhood, and also receives messages
from its neighbors: $m_{t}\left(v'\right)$ for all $v'\in\mathcal{N}\left(v\right)$.
Our goal is to minimize, for each $v\in V$, its\emph{ expected regret}
over $T$ steps:
\[
R_{T}\left(v\right)=\mathbb{E}\left[\sum_{t=1}^{T}\ell_{t}\left(I_{t}\left(v\right)\right)-\min_{i\in A}\sum_{t=1}^{T}\ell_{t}\left(i\right)\right].
\]

A well-known policy to update $\boldsymbol{p}_{t}^{v}$ is the exponential-weights
algorithm (Exp3) with weights $w_{t}^{v}\left(i\right)$ for all $i\in A$,
such that $p_{t}^{v}\left(i\right)=\frac{w_{t}^{v}\left(i\right)}{W_{t}^{v}}$
where $W_{t}^{v}=\sum_{i\in A}w_{t}^{v}\left(i\right)$ (see, e.g.,
\citep{cesa2006prediction}). The weights are updated as follows:
let $B_{t}^{v}\left(i\right)$ be the event that $v$ observed the
loss of action $i$ at step $t$; in our case $B_{t}^{v}\left(i\right)=\mathbb{I}\left\{ \exists v'\in\mathcal{N}\left(v\right):I_{t}\left(v'\right)=i\right\} $,
where $\mathbb{I}$ is the indicator function. Also, let $\hat{\ell}_{t}^{v}\left(i\right)=\frac{\ell_{t}\left(i\right)}{\mathbb{E}_{t}\left[B_{t}^{v}\left(i\right)\right]}B_{t}^{v}\left(i\right)$
be an unbiased estimated loss of action $i$ at step $t$, where $\mathbb{E}_{t}\left[\cdot\right]$
is the expectation conditioned on all the agents' choices up to step
$t$ (hence, $\mathbb{E}_{t}\left[\hat{\ell}_{t}^{v}\left(i\right)\right]=\ell_{t}\left(i\right)$).
Then
\[
w_{t+1}^{v}\left(i\right)=w_{t}^{v}\left(i\right)\exp\left(-\eta\left(v\right)\hat{\ell}_{t}^{v}\left(i\right)\right),
\]
where $\eta\left(v\right)$ is a positive parameter chosen by $v$,
called the \emph{learning rate} of agent $v$. Exp3 is given explicitly
in the appendix.
Notice that in our setting all agents $v\in V$ have the information
needed to compute $\hat{\ell}_{t}^{v}\left(i\right)$, since
\[
\mathbb{E}_{t}\left[B_{t}^{v}\left(i\right)\right]=\Pr\left[\exists v'\in\mathcal{N}\left(v\right):I_{t}\left(v'\right)=i\right]=1-\prod_{v'\in\mathcal{N}\left(v\right)}\left(1-p_{t}^{v'}\left(i\right)\right),
\]
and if agent $v$ does not observe $\ell_{t}\left(i\right)$, then
$\hat{\ell}_{t}^{v}\left(i\right)=0$.

We proceed with two useful lemmas that will help us later. For completeness,
we provide their proofs in the appendix
as well. The first lemma is the usual analysis of the exponential-weights
algorithm:
\begin{restatable}{lem}{RESTATEexpCoopBoundRegret}\label{lem:expCoopBoundRegret}
Assuming agent $v$ uses the exponential-weights
algorithm, its expected regret satisfies
\[
R_{T}\left(v\right)\le\frac{\ln K}{\eta\left(v\right)}+\frac{\eta\left(v\right)}{2}\mathbb{E}\left[\sum_{t=1}^{T}\sum_{i=1}^{K}p_{t}^{v}\left(i\right)\hat{\ell}_{t}^{v}\left(i\right)^{2}\right].
\]
\end{restatable}

The next lemma is from \citep{cesa2019delay}, and it bounds the change
of the action distribution in the exponential-weights algorithm.
\begin{restatable}{lem}{RESTATEexpCoopBoundProbabilityDifference}\label{lem:expCoopBoundProbabilityDifference}
Assuming agent $v$
uses the exponential-weights algorithm with a learning rate $\eta\left(v\right)\le\frac{1}{2K}$,
then for all $i\in A$:
\[
\left(1-\eta\left(v\right)\hat{\ell}_{t}^{v}\left(i\right)\right)p_{t}^{v}\left(i\right)\le p_{t+1}^{v}\left(i\right)\le2p_{t}^{v}\left(i\right).
\]
\end{restatable}

Also, the following definition will be needed for our algorithm. We
denote by $G^{r}$ the $r$-th power of $G$, in which $v_{1},v_{2}\in V$
are adjacent if and only if $\mathrm{dist}_{G}\left(v,v'\right)\le r$;
and by $G_{|U}$ the sub-graph of $G$ induced by $U\subseteq V$.
\begin{defn}\label{def:independentSet}
Let $G=\left\langle V,E\right\rangle $
be an undirected connected graph and let $W\subseteq U\subseteq V$.
$W$ is called an \emph{$r$-independent set} of $G$, if it is an
independent set of $G^{r}$. Namely,
\[
\forall w,w'\in W:\mathrm{dist}_{G}\left(w,w'\right)\ge r+1.
\]
If $W$ is also a maximal independent set of $\left(G^{r}\right)_{|U}$,
it is called a \emph{maximal $r$-independent subset} ($r$-MIS) of
$U$. Namely, there is no $r$-independent set $W'\subseteq U$ such
that $W\subset W'$.
\end{defn}

\section{\label{sec:centerBased}Center-based cooperative multi-armed bandits}

We now present the center-based policy for the cooperative multi-armed
bandit setting, which will give us the desired low individual regret.
In the center-based cooperative MAB, not all the agents behave similarly.
We partition the agents to three different types.

\emph{Center agents} are the agents that determine the action distribution
for all other agents. They work together with their neighbors to minimize
their regret. The neighbors of the center agents in the communication
graph, \emph{center-adjacent} agents, always copy the action distribution
from their neighboring center, and thus the centers gain more information
about their own distribution each step.

Other (not center or center-adjacent) agents are \emph{simple agents},
which simply copy the action distribution from one of the centers.
Since they are not center-adjacent, they receive the action distribution
with delay, through other agents that copy from the same center.

We artificially partition the graph to connected components, such
that each center $c$ has its own component, and all the simple agents
in the component of $c$ copy their action distribution from it. To
obtain a low individual regret, we require the components to have
a relatively small diameter, and the center agents to have a high
degree in the communication graph. Namely, center agents have the
highest or nearly highest degree in their component.

In more detail, we select a set $C\subseteq V$ of center\emph{ }agents.
All center agents $c\in C$ use the exponential-weights algorithm
with a learning rate $\eta\left(c\right)=\frac{1}{2}\sqrt{\frac{\left(\ln K\right)\min\left\{ \left|\mathcal{N}\left(c\right)\right|,K\right\} }{KT}}$.
The agent set $V$ is partitioned into disjoint subsets $\left\{ V_{c}\subseteq V\mid c\in C\right\} $,
such that $\mathcal{N}\left(c\right)\subseteq V_{c}$ for all $c\in C$,
and the sub-graph $G_{c}\equiv G_{|V_{c}}$ induced by $V_{c}$ is
connected. Notice that since the components are disjoint, the condition
$\mathcal{N}\left(c\right)\subseteq V_{c}$ implies $C$ is a $2$-independent
set. For all non-centers $v\in V\setminus C$, we denote by $\mathcal{C}\left(v\right)\in C$
the\emph{ }center agent such that $v\in V_{\mathcal{C}\left(v\right)}$,
and call it the \emph{center of $v$}. All non-center agents $v\in V\setminus C$
copy their distribution from their \emph{origin neighbor $U\left(v\right)$,
}which is their neighbor in $G_{\mathcal{C}\left(v\right)}$ closest
to $\mathcal{C}\left(v\right)$, breaking ties arbitrarily. Namely,
\[
U\left(v\right)=\argmin_{v'\in\mathcal{N}\left(v\right)\cap V_{\mathcal{C}\left(v\right)}}\mathrm{dist}_{G_{\mathcal{C}\left(v\right)}}\left(v',\mathcal{C}\left(v\right)\right).
\]
Thus, agent $v$ receives its center's distribution with a delay of
$d\left(v\right)=\mathrm{dist}_{G_{\mathcal{C}\left(v\right)}}\left(v,\mathcal{C}\left(v\right)\right)$
steps, so for all $t\ge d\left(v\right)+1$:
\[
\boldsymbol{p}_{t}^{v}=\boldsymbol{p}_{t-d\left(v\right)}^{\mathcal{C}\left(v\right)}.
\]
Notice that if $v\in\mathcal{N}\left(c\right)$, then $v$ is center-adjacent
and it holds $U\left(v\right)=\mathcal{C}\left(v\right)$ and $d\left(v\right)=1$.
For completeness, we define $U\left(c\right)=\mathcal{C}\left(c\right)=c$
and $d\left(c\right)=0$ for all $c\in C$.

To express the regret of the center-based policy, we introduce a new
concept:
\begin{defn}\label{def:mass}
The \emph{mass} of a center agent $c\in C$ is defined
to be
\[
M\left(c\right)\equiv\min\left\{ \left|\mathcal{N}\left(c\right)\right|,K\right\} ,
\]
and the mass of non-center agent $v\in V\setminus C$ is
\[
M\left(v\right)\equiv e^{-\frac{1}{6}d\left(v\right)}M\left(\mathcal{C}\left(v\right)\right).
\]
Notice the mass depends only on how the graph is partitioned, and
it satisfies $M\left(v\right)=e^{-\frac{1}{6}}M\left(U\left(v\right)\right)$
for all non-centers $v\in V\setminus C$. Intuitively, the mass of
agent $v$ captures the idea that as the degree of the center is larger
and as the agent is closer to its center, the lower the regret of
$v$. We prove that the regret is $\widetilde{O}\left(\sqrt{\frac{K}{M\left(v\right)}T}\right)$.
Our partitioning algorithms, presented in the next section, show that
the mass of agent $v$ satisfies $M\left(v\right)=\Omega\left(\min\left\{ \left|\mathcal{N}\left(v\right)\right|,K\right\} \right)$,
so we obtain an individual regret of the order of $\widetilde{O}\left(\sqrt{\left(1+\frac{K}{\left|\mathcal{N}\left(v\right)\right|}\right)T}\right)$.

We specify the center-based policy in Algorithms \ref{alg:center-based-policy-centers}
and \ref{alg:center-based-policy-non-centers}. We emphasize that
before the agents use the center-based policy they must partition
the graph with one of the algorithms we present in the next section.
While the agents partition the graph, they play arbitrary actions.
\begin{algorithm}
\caption{\label{alg:center-based-policy-centers}Center-based cooperative MAB
- $v$ is a center agent}
\begin{algorithmic}[1]

\Parameters{Number of arms $K$; Time horizon $T$.}

\Init{$\eta\left(v\right)\leftarrow\frac{1}{2}\sqrt{\frac{\left(\ln K\right)M\left(v\right)}{KT}}$;
$w_{1}^{v}\left(i\right)\leftarrow\frac{1}{K}$ for all $i\in A$.}

\For{$t\le T$}

\State{Set $p_{t}^{v}\left(i\right)\leftarrow\frac{w_{t}^{v}\left(i\right)}{W_{t}^{v}}$
for all $i\in A$, where $W_{t}^{v}=\sum_{i\in A}w_{t}^{v}\left(i\right)$.}

\State{Play an action $I_{t}\left(v\right)$ drawn from $\boldsymbol{p}_{t}^{v}=\left\langle p_{t}^{v}\left(1\right),\dots,p_{t}^{v}\left(K\right)\right\rangle $.}

\State{Observe loss $\ell_{t}\left(I_{t}\left(v\right)\right)$.}

\State{Send the following message to the set $\mathcal{N}\left(v\right)$:
$m_{t}\left(v\right)=\left\langle v,t,I_{t}\left(v\right),\ell_{t}\left(I_{t}\left(v\right)\right),\boldsymbol{p}_{t}^{v}\right\rangle $.}

\State{Receive all messages $m_{t}\left(v'\right)$ from $v'\in\mathcal{N}\left(v\right)$.}

\State{Update for all $i\in A$: $w_{t+1}^{v}\left(i\right)\leftarrow w_{t}^{v}\left(i\right)\exp\left(-\eta\left(v\right)\hat{\ell}_{t}^{v}\left(i\right)\right)$,
where
\[
\hat{\ell}_{t}^{v}\left(i\right)=\frac{\ell_{t}\left(i\right)}{\mathbb{E}_{t}\left[B_{t}^{v}\left(i\right)\right]}B_{t}^{v}\left(i\right),
\]
\[
B_{t}^{v}\left(i\right)=\mathbb{I}\left\{ \exists v'\in\mathcal{N}\left(v\right):I_{t}\left(v'\right)=i\right\} ,\hspace{1em}\mathbb{E}_{t}\left[B_{t}^{v}\left(i\right)\right]=1-\prod_{v'\in\mathcal{N}\left(v\right)}\left(1-p_{t}^{v'}\left(i\right)\right).
\]
}

\EndFor

\end{algorithmic}
\end{algorithm}
\begin{algorithm}
\caption{\label{alg:center-based-policy-non-centers}Center-based cooperative
MAB - $v$ is a non-center agent}
\begin{algorithmic}[1]

\Parameters{Number of arms $K$; Time horizon $T$; Origin neighbor
$U\left(v\right)$.}

\Init{$p_{1}^{v}\left(i\right)\leftarrow\frac{1}{K}$ for all $i\in A$.}

\For{$t\le T$}

\State{Play an action $I_{t}\left(v\right)$ drawn from $\boldsymbol{p}_{t}^{v}=\left\langle p_{t}^{v}\left(1\right),\dots,p_{t}^{v}\left(K\right)\right\rangle $.}

\State{Observe loss $\ell_{t}\left(I_{t}\left(v\right)\right)$.}

\State{Send the following message to the set $\mathcal{N}\left(v\right)$:
$m_{t}\left(v\right)=\left\langle v,t,I_{t}\left(v\right),\ell_{t}\left(I_{t}\left(v\right)\right),\boldsymbol{p}_{t}^{v}\right\rangle $.}

\State{Receive the message $m_{t}\left(U\left(v\right)\right)$ from
$U\left(v\right)$.}

\State{Update $p_{t+1}^{v}\left(i\right)=p_{t}^{U\left(v\right)}\left(i\right)$
for all $i\in A$.}

\EndFor

\end{algorithmic}
\end{algorithm}
\end{defn}

\section{\label{sec:partitioningTheGraph}Partitioning the graph}

The goal now is to show that we can partition the graph such that
the mass is large for every $v\in V$. In particular, we want to show
that any graph can be partitioned such that $M\left(v\right)=\Omega\left(\min\left\{ \left|\mathcal{N}\left(v\right)\right|,K\right\} \right)$.

We consider two cases: the informed and uninformed settings. In the
informed setting, all of the agents have access to the graph structure.
Each agent can partition the graph by itself in advance, to know the
role it plays: whether it is a center or not, and which agent is its
origin neighbor. In the uninformed setting, the graph structure is
not known to the agents, only their neighbors and an upper bound on
the total number of agents $\bar{N}\ge N$. The agents partition the
graph using a distributed algorithm while playing actions and suffering
loss.

The basic structure of the partitioning algorithm in both settings
is the same. First, we show an algorithm that computes the connected
components given a center set $C$. Then, we show an algorithm that
computes a center set $C$. The second algorithm is specifically designed
to be used with the first, and together they partition the graph to
connected components such that every agent has a large mass.

\subsection{Computing graph components given a center set}

Given a center set $C$, we show a distributed algorithm called \emph{Centers-to-Components},
which computes the connected components, and present it in Algorithm
\ref{alg:centersToComponents}. Although it is distributed, in the
informed setting agents can simply simulate it locally in advance.

Centers-to-Components runs simultaneous distributed BFS graph traversals,
originating from every center $c\in C$. When the traversal of center
$c$ arrives to a simple agent $v\in V\setminus C$, $v$ decides
if $c$ is the best center for it so far, and if it is, $v$ switches
its component to $V_{c}$. Notice each agent needs to know only if
itself is a center or not.
\begin{algorithm}
\caption{\label{alg:centersToComponents}Centers-to-Components}
\begin{algorithmic}[1]

\Parameters{Number of arms $K$; Center set $C$.}

\Init{Number of iterations $\Theta_{K}\leftarrow\left\lfloor 12\ln K\right\rfloor $.}

\If{$v\in C$}

\State{Initialize: $\mathcal{C}_{0}\left(v\right)\leftarrow v;\hspace{1em}U_{0}\left(v\right)\leftarrow v;\hspace{1em}M_{0}\left(v\right)\leftarrow\min\left\{ \left|\mathcal{N}\left(v\right)\right|,K\right\} $.}

\Else

\State{Initialize: $\mathcal{C}_{0}\left(v\right)\leftarrow\mathrm{nil};\hspace{1em}U_{0}\left(v\right)\leftarrow\mathrm{nil};\hspace{1em}M_{0}\left(v\right)\leftarrow0$.}

\EndIf

\For{$0\le t\le\Theta_{K}$}

\State{Send the following message to the set $\mathcal{N}\left(v\right)$:
$\mu_{t}\left(v\right)=\left\langle v,t,\mathcal{C}_{t}\left(v\right),M_{t}\left(v\right)\right\rangle $.}

\State{Receive all messages $\mu_{t}\left(v'\right)$ from $v'\in\mathcal{N}\left(v\right)$.}

\If{$U_{t}\left(v\right)\notin C$ }\Comment{The center-based policy
requires $\mathcal{N}\left(c\right)\subseteq V_{c}$ for all $c\in C$.}

\State{Find the best origin neighbor for $v$:
\[
U_{t+1}\left(v\right)\leftarrow\argmax_{v'\in\mathcal{N}\left(v\right)\setminus\left\{ v\right\} }M_{t}\left(v'\right).
\]
}

\State{Update: $\mathcal{C}_{t+1}\left(v\right)\leftarrow\mathcal{C}_{t}\left(U_{t+1}\left(v\right)\right);\hspace{1em}M_{t+1}\left(v\right)\leftarrow e^{-\frac{1}{6}}M_{t}\left(U_{t+1}\left(v\right)\right)$.}

\Else

\State{Keep old values: $\mathcal{C}_{t+1}\left(v\right)\leftarrow\mathcal{C}_{t}\left(v\right);\hspace{1em}U_{t+1}\left(v\right)\leftarrow U_{t}\left(v\right);\hspace{1em}M_{t+1}\left(v\right)\leftarrow M_{t}\left(v\right)$.}

\EndIf

\EndFor

\State{\Return{
\[
\mathcal{C}\left(v\right)=\mathcal{C}_{\Theta_{K}+1}\left(v\right);\hspace{1em}U\left(v\right)=U_{\Theta_{K}+1}\left(v\right);\hspace{1em}M\left(v\right)=M_{\Theta_{K}+1}\left(v\right).
\]
}}

\end{algorithmic}
\end{algorithm}

\subsection{Computing centers}

To compute the center set $C$, we show two algorithms; one for the
informed setting and one for the uninformed setting. The regret bound
for the informed setting is slightly better, and the algorithm is
simpler.

\paragraph{The informed setting}

The algorithm that computes the center set in the informed setting
is called \emph{Compute-Centers-Informed} and is presented in Algorithm
\ref{alg:comupteCentersInformed}. The center set is built in a greedy
way: each iteration, all of the agents test if they are ``satisfied''
with the current center set (i.e., $M\left(v\right)\ge\min\left\{ \left|\mathcal{N}\left(v\right)\right|,K\right\} $).
If there are unsatisfied agents left, the agent with the highest degree
is added to the center set.
\begin{algorithm}
\caption{\label{alg:comupteCentersInformed}Compute-Centers-Informed}
\begin{algorithmic}[1]

\Parameters{Undirected connected graph $G=\left\langle V,E\right\rangle $;
Number of arms $K$.}

\Init{Center set $C_{0}\leftarrow\emptyset$; Unsatisfied agents
$S_{0}\leftarrow V$.}

\State{$t\leftarrow0$.}

\While{$S_{t}\neq\emptyset$}

\State{Choose the next center: $c_{t}\leftarrow\argmax_{v\in S_{t}}\left|\mathcal{N}\left(v\right)\right|$.}

\State{Update $C_{t+1}\leftarrow C_{t}\cup\left\{ c_{t}\right\} $.}

\State{Run Centers-to-Components with center set $C_{t+1}$, and
obtain mass $M_{t+1}\left(v\right)$ for each $v\in V$.}

\State{Update
\[
S_{t+1}\leftarrow\left\{ v\in V\mid M_{t+1}\left(v\right)<\min\left\{ \left|\mathcal{N}\left(v\right)\right|,K\right\} \land\min_{c\in C_{t+1}}\mathrm{dist}_{G}\left(v,c\right)\ge3\right\} .
\]
}

\State{$t\leftarrow t+1$.}

\EndWhile

\State{\Return{$C=C_{t}$.}}

\end{algorithmic}
\end{algorithm}

\paragraph{The uninformed setting}

At first, it may seem that the uninformed setting can be solved the
same way as the informed setting, with some distributed version of
Compute-Centers-Informed. However, such algorithm will require $\Omega\left(N\right)$
steps in the worst case, since at each iteration only one agent becomes
a center. In the informed setting we do not care about this, since
the components are computed in advance. In the uninformed setting
however, at each step of the algorithm the agents suffer a loss, and
thus the regret bound will be at least linear in the number of agents,
which can be very large.

To avoid this problem, we need to add many centers each iteration,
and not just one as in Compute-Centers-Informed. To do this, we exploit
the fact that there are only $K$ possible values for a center's mass.
In our algorithm, there are $K$ iterations, and in each iteration
$t$, as many agents as possible with degree $K-t$ become centers.
To ensure the final center set is 2-independent, only a 2-MIS of the
potential center agents are added to the center set each iteration.

To compute a 2-MIS in a distributed manner, we use Luby's algorithm
\citep{luby1986simple,alon1986fast} on the sub-graph of $G^{2}$
induced by the potential center agents. Briefly, at each iteration
of Luby's algorithm, every potential center agent picks a number uniformly
from $\left[0,1\right]$. Agents that picked the maximal number among
their neighbors of distance 2 join the 2-MIS, and their neighbors
of distance 2 stop participating. A 2-MIS is computed after $\left\lceil 3\ln\left(\frac{N}{\sqrt{\delta}}\right)\right\rceil $
iterations with probability $1-\delta$. Each iteration requires exchanging
4 messages - 2 for communicating the random numbers and 2 for communicating
the new agents in the 2-MIS. Hence, $4\left\lceil 3\ln\left(\frac{N}{\sqrt{\delta}}\right)\right\rceil $
steps suffice to compute a 2-MIS with probability $1-\delta$. A more
detailed explanation of Luby's algorithm can be found in the appendix.

We present \emph{Compute-Centers-Uninformed} in Algorithm \ref{alg:computeCentersUninformed}.
Since this is a distributed algorithm, we have the variables $\mathbb{C}\left(v\right)$
and $\mathbb{S}\left(v\right)$ as indicators for whether $v$ is
a center or unsatisfied, respectively.
\begin{algorithm}
\caption{\label{alg:computeCentersUninformed}Compute-Centers-Uninformed -
agent $v$}
\begin{algorithmic}[1]

\Parameters{Number of arms $K$; Upper bound on the total number
of agents $\bar{N}$; Time horizon $T$.}

\Init{Center indicator $\mathbb{C}\left(v\right)\leftarrow\mathrm{FALSE}$;
Unsatisfied indicator $\mathbb{S}\left(v\right)\leftarrow\mathrm{TRUE}$.}

\For{$0\le t\le K-1$}

\State{Participate for $4\left\lceil 3\ln\left(\bar{N}\sqrt{KT}\right)\right\rceil $
steps in Luby's algorithm on $\left(G^{2}\right)_{|S_{t}}$, where
\[
S_{t}=\left\{ v\in V\mid\mathbb{S}\left(v\right)=\mathrm{TRUE}\land\min\left\{ \left|\mathcal{N}\left(v\right)\right|,K\right\} =K-t\right\} ,
\]
to compute $W_{t}$, a 2-MIS of $S_{t}$, with probability $1-\frac{1}{TK}$.}

\State{If $v\in W_{t}$, set $\mathbb{C}\left(v\right)\leftarrow\mathrm{TRUE}$.}

\State{Participate in Centers-to-Components with center set $C_{t}=\left\{ v'\in V\mid\mathbb{C}\left(v'\right)=\mathrm{TRUE}\right\} $;
obtain mass $M_{t}\left(v\right)$ and whether $\min_{c\in C_{t}}\mathrm{dist}_{G}\left(v,c\right)\ge3$.}

\State{\Comment{$\min_{c\in C_{t}}\mathrm{dist}_{G}\left(v,c\right)\ge3$
if and only if $\mathcal{C}_{2}\left(v\right)=\mathrm{nil}$ in Centers-to-Components.}}

\State{Update
\[
\mathbb{S}\left(v\right)\leftarrow\mathbb{I}\left[M_{t}\left(v\right)<\min\left\{ \left|\mathcal{N}\left(v\right)\right|,K\right\} \land\min_{c\in C_{t}}\mathrm{dist}_{G}\left(v,c\right)\ge3\right].
\]
}

\EndFor

\State{\Return{$C=C_{K-1}$.}}

\end{algorithmic}
\end{algorithm}

\section{\label{sec:regretAnalysis}Regret analysis}

We will now provide an analysis of our algorithms.

\subsection{\label{subsec:analyzingCenterBased}Individual regret of the center-based
policy}

We start by bounding the expected regret of the center agents, when
they are using the center-based policy.
\begin{restatable}{lem}{RESTATEcentersRegret}\label{lem:centersRegret}
Let $T\ge K^{2}\ln K$. Using the center-based
policy, the expected regret of each center $c\in C$ satisfies
\[
R_{T}\left(c\right)\le4\sqrt{\left(\ln K\right)\frac{K}{M\left(c\right)}T}.
\]
\end{restatable}

\begin{proof}
Since $T\ge K^{2}\ln K$, we have $\eta\left(c\right)=\frac{1}{2}\sqrt{\frac{\left(\ln K\right)M\left(\text{c}\right)}{KT}}\le\frac{1}{2}\sqrt{\frac{M\left(c\right)}{K^{3}}}\le\frac{1}{2K}$.
Hence, from Lemma \ref{lem:expCoopBoundProbabilityDifference}, we
get for any $v\in\mathcal{N}\left(c\right)\setminus c$ and $i\in A$:
\[
p_{t}^{v}\left(i\right)=p_{t-1}^{c}\left(i\right)\ge\frac{1}{2}p_{t}^{c}\left(i\right).
\]
Hence,
\begin{align*}
\mathbb{E}_{t}\left[B_{t}^{c}\left(i\right)\right] & =1-\prod_{v\in\mathcal{N}\left(c\right)}\left(1-p_{t}^{v}\left(i\right)\right)\\
 & \ge1-\left(1-\frac{1}{2}p_{t}^{c}\left(i\right)\right)^{\left|\mathcal{N}\left(c\right)\right|}\\
 & \ge1-\exp\left(-\frac{1}{2}p_{t}^{c}\left(i\right)\left|\mathcal{N}\left(c\right)\right|\right) & (1-x\le e^{-x})\\
 & \ge1-\exp\left(-\min\left\{ \frac{1}{2}\left|\mathcal{N}\left(c\right)\right|p_{t}^{c}\left(i\right),1\right\} \right)\\
 & \ge\left(1-e^{-1}\right)\min\left\{ \frac{1}{2}\left|\mathcal{N}\left(c\right)\right|p_{t}^{c}\left(i\right),1\right\} , & (\left(1-e^{-1}\right)x\le1-e^{-x}\text{ for }0\le x\le1)
\end{align*}
and thus,
\begin{align*}
\mathbb{E}_{t}\left[\hat{\ell}_{t}^{c}\left(i\right)^{2}\right] & =\mathbb{E}_{t}\left[\frac{\ell_{t}\left(i\right)^{2}}{\mathbb{E}_{t}\left[B_{t}^{c}\left(i\right)\right]^{2}}B_{t}^{c}\left(i\right)\right]\\
 & \le\frac{1}{\mathbb{E}_{t}\left[B_{t}^{c}\left(i\right)\right]} & (\ell_{t}\left(i\right)\le1)\\
 & \le\frac{1}{\left(1-e^{-1}\right)\min\left\{ \frac{1}{2}\left|\mathcal{N}\left(c\right)\right|p_{t}^{c}\left(i\right),1\right\} }\\
 & \le2+\frac{4}{\left|\mathcal{N}\left(c\right)\right|p_{t}^{c}\left(i\right)}.
\end{align*}
By Lemma \ref{lem:expCoopBoundRegret}, we now obtain
\begin{align*}
R_{T}\left(c\right) & \le\frac{\ln K}{\eta\left(c\right)}+\frac{\eta\left(c\right)}{2}\mathbb{E}\left[\sum_{t=1}^{T}\sum_{i=1}^{K}p_{t}^{c}\left(i\right)\mathbb{E}_{t}\left[\hat{\ell}_{t}^{c}\left(i\right)^{2}\right]\right]\\
 & \le\frac{\ln K}{\eta\left(c\right)}+\frac{\eta\left(c\right)}{2}\left(2+4\frac{K}{\left|\mathcal{N}\left(c\right)\right|}\right)T\\
 & \le\frac{\ln K}{\eta\left(c\right)}+4\eta\left(c\right)\frac{K}{M\left(c\right)}T\\
 & =4\sqrt{\left(\ln K\right)\frac{K}{M\left(c\right)}T} & (\eta\left(c\right)=\frac{1}{2}\sqrt{\frac{\left(\ln K\right)M\left(c\right)}{KT}})
\end{align*}
as claimed.
\end{proof}
Since non-center agents use the same distribution as some center,
only with delay, we can use this result together with Lemma \ref{lem:expCoopBoundProbabilityDifference}
to bound the regret of all agents in the graph.
\begin{restatable}{thm}{RESTATEnonCenterRegret}\label{thm:nonCenterRegret}
Let $T\ge K^{2}\ln K$. Using the center-based
policy, the regret of each agent $v\in V$ satisfies
\[
R_{T}\left(v\right)\le7\sqrt{\left(\ln K\right)\frac{K}{M\left(v\right)}T}.
\]
\end{restatable}

\begin{proof}
Again, since $T\ge K^{2}\ln K$ we have $\eta\left(v\right)\le\frac{1}{2K}$.
Recall that $\boldsymbol{p}_{t}^{v}=\boldsymbol{p}_{t-d\left(v\right)}^{\mathcal{C}\left(v\right)}$.
Thus, we can use Lemma \ref{lem:expCoopBoundProbabilityDifference}
iteratively to obtain for all $t>d\left(v\right)$:
\begin{align*}
p_{t}^{v}\left(i\right) & =p_{t-d\left(v\right)}^{\mathcal{C}\left(v\right)}\left(i\right)\\
 & \le p_{t-d\left(v\right)+1}^{\mathcal{C}\left(v\right)}\left(i\right)+\eta\left(\mathcal{C}\left(v\right)\right)p_{t-d\left(v\right)}^{\mathcal{C}\left(v\right)}\left(i\right)\hat{\ell}_{t-d\left(v\right)}^{\mathcal{C}\left(v\right)}\left(i\right)\\
 & \le\cdots\le p_{t}^{\mathcal{C}\left(v\right)}\left(i\right)+\eta\left(\mathcal{C}\left(v\right)\right)\sum_{s=1}^{d\left(v\right)}p_{t-s}^{\mathcal{C}\left(v\right)}\left(i\right)\hat{\ell}_{t-s}^{\mathcal{C}\left(v\right)}\left(i\right),
\end{align*}
which yields
\begin{align*}
R_{T}\left(v\right) & =\mathbb{E}\left[\sum_{t=1}^{T}\ell_{t}\left(I_{t}\left(v\right)\right)-\min_{i\in A}\sum_{t=1}^{T}\ell_{t}\left(i\right)\right]\\
 & =\mathbb{E}\left[\sum_{t=1}^{T}\sum_{i=1}^{K}p_{t}^{v}\left(i\right)\ell_{t}\left(i\right)-\min_{i\in A}\sum_{t=1}^{T}\ell_{t}\left(i\right)\right]\\
 & \le d\left(v\right)+\mathbb{E}\left[\sum_{t=d\left(v\right)}^{T}\sum_{i=1}^{K}p_{t}^{\mathcal{C}\left(v\right)}\left(i\right)\ell_{t}\left(i\right)-\min_{i\in A}\sum_{t=d\left(v\right)}^{T}\ell_{t}\left(i\right)\right]\\
 & +\eta\left(\mathcal{C}\left(v\right)\right)\mathbb{E}\left[\sum_{t=d\left(v\right)}^{T}\sum_{i=1}^{K}\sum_{s=1}^{d\left(v\right)}p_{t-s}^{\mathcal{C}\left(v\right)}\left(i\right)\hat{\ell}_{t-s}^{\mathcal{C}\left(v\right)}\left(i\right)\ell_{t}\left(i\right)\right]\\
 & \le R_{T}\left(\mathcal{C}\left(v\right)\right)+d\left(v\right)+d\left(v\right)\eta\left(\mathcal{C}\left(v\right)\right)T,
\end{align*}
where the last inequality is implied from
\[
\mathbb{E}_{t-s}\left[\sum_{i=1}^{K}p_{t-s}^{\mathcal{C}\left(v\right)}\left(i\right)\hat{\ell}_{t-s}^{\mathcal{C}\left(v\right)}\left(i\right)\ell_{t}\left(i\right)\right]\le\sum_{i=1}^{K}p_{t-s}^{\mathcal{C}\left(v\right)}\left(i\right)\ell_{t-s}\left(i\right)\le1.
\]
Hence, using Lemma \ref{lem:centersRegret} we get
\begin{align*}
R_{T}\left(v\right) & \le\left(4\sqrt{\frac{K}{M\left(\mathcal{C}\left(v\right)\right)}}+d\left(v\right)+\frac{d\left(v\right)}{2}\sqrt{\frac{M\left(\mathcal{C}\left(v\right)\right)}{K}}\right)\sqrt{\left(\ln K\right)T}\\
 & \le\left(4\sqrt{\frac{K}{M\left(\mathcal{C}\left(v\right)\right)}}+d\left(v\right)\sqrt{\frac{M\left(\mathcal{C}\left(v\right)\right)}{K}}\right)\sqrt{\left(\ln K\right)T}\\
 & \le7\sqrt{\left(\ln K\right)\frac{K}{M\left(v\right)}T},
\end{align*}
concluding our proof.
\end{proof}
This individual regret bound holds simultaneously for all agents in
the graph, and it depends only on the graph structure and components.

\subsection{\label{subsec:analyzingCentersToComponents}Analyzing Centers-to-Components}

We need to show the results of Centers-to-Components follow their
definitions, and the derived components satisfy all the properties
required by the center-based policy. We first show two helpful lemmas
that analyze the results of Centers-to-Components. In the following,
we denote $\Theta_{K}=\left\lfloor 12\ln K\right\rfloor $ and $\tau^{v}=\min_{c\in C}\mathrm{dist}_{G}\left(v,c\right)$
for any $v\in V$.
\begin{restatable}{lem}{RESTATEcentersToComponentsHelper}\label{lem:centersToComponentsHelper}
Let $\mathcal{C}_{t}\left(v\right),U_{t}\left(v\right),M_{t}\left(v\right)$
be the variables of agent $v$ at iteration $t$ from Centers-to-Components.
Then the following properties hold for all $1\le t\le\Theta_{K}+1$
and $v\in V\setminus C$:
\begin{enumerate}
\item $M_{t}\left(v\right)\ge M_{t-1}\left(v\right)$.
\item If $M_{t}\left(v\right)\neq M_{t-1}\left(v\right)$, then $\mathcal{C}_{t}\left(v\right)\neq\mathrm{nil}$
and
\[
M_{t}\left(v\right)=e^{-\frac{1}{6}t}M\left(\mathcal{C}_{t}\left(v\right)\right).
\]
Moreover, $t\ge\mathrm{dist}_{G}\left(v,\mathcal{C}_{t}\left(v\right)\right)$.
\item If $\tau^{v}\le\Theta_{K}+1$, then $M_{\tau^{v}}\left(v\right)\ge e^{-\frac{1}{6}\tau^{v}}$.
\item If $\tau^{v}\le6\ln K$, then $\mathcal{C}_{\Theta_{K}+1}\left(v\right)=\mathcal{C}_{\Theta_{K}}\left(v\right),M_{\Theta_{K}+1}\left(v\right)=M_{\Theta_{K}}\left(v\right)$.
\end{enumerate}
\end{restatable}

\begin{proof}
$ $
\begin{enumerate}
\item For center-adjacent agents this is immediate from the algorithm. Otherwise,
let $v$ be a simple agent. We proceed by induction over $t$. If
$t=1$, we have $M_{t}\left(v\right)=M_{1}\left(v\right)=M_{0}\left(v\right)=0$.
Assume for all $t>1$ and $v'\in V\setminus C$ that $M_{t-1}\left(v'\right)\ge M_{t-2}\left(v'\right)$.
Since we choose $U_{t}\left(v\right)$ to be the neighbor with maximal
mass at iteration $t-1$, for any $t>1$ we get
\begin{align*}
M_{t}\left(v\right) & =e^{-\frac{1}{6}}M_{t-1}\left(U_{t}\left(v\right)\right)\\
 & \ge e^{-\frac{1}{6}}M_{t-1}\left(U_{t-1}\left(v\right)\right)\\
 & \ge e^{-\frac{1}{6}}M_{t-2}\left(U_{t-1}\left(v\right)\right)\\
 & =M_{t-1}\left(v\right)
\end{align*}
as desired.
\item Again we proceed by induction on $t$. If $t=1$ and $M_{t}\left(v\right)=M_{1}\left(v\right)\neq M_{0}\left(v\right)=0$,
then $M_{0}\left(U_{1}\left(v\right)\right)\neq0$, and thus $U_{1}\left(v\right)=\mathcal{C}_{1}\left(v\right)\in C$.
Hence, $M_{t}\left(v\right)=M_{1}\left(v\right)=e^{-\frac{1}{6}t}M\left(\mathcal{C}_{1}\left(v\right)\right)$.
For any $t>1$, we assume the property is true for any $v'\in V\setminus C$
at iteration $t-1$. If $M_{t}\left(v\right)\neq M_{t-1}\left(v\right)$
we obtain from property 1 that $M_{t}\left(v\right)>M_{t-1}\left(v\right)$,
and thus
\[
e^{-\frac{1}{6}}M_{t-1}\left(U_{t}\left(v\right)\right)>e^{-\frac{1}{6}}M_{t-2}\left(U_{t-1}\left(v\right)\right).
\]
From the way $U_{t-1}\left(v\right)$ is chosen we get
\[
M_{t-1}\left(U_{t}\left(v\right)\right)>M_{t-2}\left(U_{t-1}\left(v\right)\right)\ge M_{t-2}\left(U_{t}\left(v\right)\right).
\]
Hence, $M_{t-1}\left(U_{t}\left(v\right)\right)\neq M_{t-2}\left(U_{t}\left(v\right)\right)$,
and from our assumption
\[
M_{t}\left(v\right)=e^{-\frac{1}{6}}M_{t-1}\left(U_{t}\left(v\right)\right)=e^{-\frac{1}{6}-\frac{1}{6}\left(t-1\right)}M\left(\mathcal{C}_{t-1}\left(U_{t}\left(v\right)\right)\right)=e^{-\frac{1}{6}t}M\left(\mathcal{C}_{t}\left(v\right)\right),
\]
where in the last equality we used the fact that $\mathcal{C}_{t}\left(v\right)=\mathcal{C}_{t-1}\left(U_{t}\left(v\right)\right)$.
We also get
\[
t-1\ge\mathrm{dist}_{G}\left(U_{t}\left(v\right),\mathcal{C}_{t-1}\left(U_{t}\left(v\right)\right)\right)=\mathrm{dist}_{G}\left(U_{t}\left(v\right),\mathcal{C}_{t}\left(v\right)\right).
\]
Hence, since $\mathrm{dist}_{G}\left(v,\mathcal{C}_{t}\left(v\right)\right)\le\mathrm{dist}_{G}\left(U_{t}\left(v\right),\mathcal{C}_{t}\left(v\right)\right)+1$,
we obtain $t\ge\mathrm{dist}_{G}\left(v,\mathcal{C}_{t}\left(v\right)\right)$
as desired.
\item We proceed by induction on $\tau^{v}$. If $\tau^{v}=\min_{c\in C}\mathrm{dist}_{G}\left(v,c\right)=1$,
then $v$ is center-adjacent and thus $\mathcal{C}_{1}\left(v\right)=\argmin_{c\in C}\mathrm{dist}_{G}\left(v,c\right)$,
which gives
\[
M_{\tau^{v}}\left(v\right)=M_{1}\left(v\right)=e^{-\frac{1}{6}}M\left(\mathcal{C}_{1}\left(v\right)\right)\ge e^{-\frac{1}{6}\tau^{v}}.
\]
Otherwise, let $v\in V\setminus C$ be a simple agent with $\tau^{v}>1$
and $c=\argmin_{c'\in C}\mathrm{dist}_{G}\left(v,c'\right)$. It must
have a neighbor $v'\in\mathcal{N}\left(v\right)$ such that $c=\argmin_{c'\in C}\mathrm{dist}_{G}\left(v',c'\right)$
and $\tau^{v'}=\mathrm{dist}_{G}\left(v',c\right)=\mathrm{dist}_{G}\left(v,c\right)-1=\tau^{v}-1$.
We assume the property is true for $v'$. From the way the origin
neighbor at iteration $\tau^{v}$ is chosen we obtain that
\begin{align*}
M_{\tau^{v}}\left(v\right) & =e^{-\frac{1}{6}}M_{\tau^{v}-1}\left(U_{\tau^{v}}\left(v\right)\right)\\
 & \ge e^{-\frac{1}{6}}M_{\tau^{v}-1}\left(v'\right)\\
 & =e^{-\frac{1}{6}}M_{\tau^{v'}}\left(v'\right)\\
 & \ge e^{-\frac{1}{6}\left(1+\tau^{v'}\right)}\\
 & =e^{-\frac{1}{6}\tau^{v}},
\end{align*}
as desired.
\item Assuming to the contrary $M_{\Theta_{K}+1}\left(v\right)\neq M_{\Theta_{K}}\left(v\right)$
(or $\mathcal{C}_{\Theta_{K}+1}\left(v\right)\neq\mathcal{C}_{\Theta_{K}}\left(v\right)$),
we get
\begin{align*}
M_{\Theta_{K}+1}\left(v\right) & =e^{-\frac{1}{6}\left(\Theta_{K}+1\right)}M\left(\mathcal{C}_{\Theta_{K}+1}\left(v\right)\right) & \text{(property 2)}\\
 & <\frac{1}{K}\\
 & \le e^{-\frac{1}{6}\tau^{v}} & (\tau^{v}\le6\ln K)\\
 & \le M_{\tau^{v}}\left(v\right) & \text{(property 3)}\\
 & \le M_{\left\lceil 6\ln K\right\rceil }\left(v\right), & \text{(property 1)}
\end{align*}
contradicting property 1 and concluding our proof.
\end{enumerate}
\end{proof}
\begin{restatable}{lem}{RESTATEifCloseThenLocallyGood}\label{lem:ifCloseThenLocallyGood}
Let $\mathcal{C}\left(v\right),U\left(v\right),M\left(v\right)$
be the results of Centers-to-Components. Then the following properties
hold for all simple agents $v\in V\setminus C$ such that $2\le\min_{c\in C}\mathrm{dist}_{G}\left(v,c\right)\le6\ln K-1$:
\begin{enumerate}
\item $\mathcal{C}\left(v\right)\neq\mathrm{\mathrm{nil}}$ and $U\left(v\right)\neq\mathrm{\mathrm{nil}}$.
\item $M\left(v\right)=e^{-\frac{1}{6}}M\left(U\left(v\right)\right)$.
\item $\mathcal{C}\left(v\right)=\mathcal{C}\left(U\left(v\right)\right)$.
\end{enumerate}
\end{restatable}

\begin{proof}
Let $v\in V$ be a simple agent such that $\tau^{v}=\min_{c\in C}\mathrm{dist}_{G}\left(v,c\right)\le6\ln K-1$.
\begin{enumerate}
\item We have
\begin{align*}
M\left(v\right) & =M_{\Theta_{K}+1}\left(v\right)\\
 & \ge M_{\tau^{v}}\left(v\right) & \text{(property 1 of Lemma \ref{lem:centersToComponentsHelper})}\\
 & >0, & \text{(property 3 of Lemma \ref{lem:centersToComponentsHelper})}
\end{align*}
and thus it follows from the algorithm that $\mathcal{C}\left(v\right)\neq\mathrm{\mathrm{nil}}$
and $U\left(v\right)\neq\mathrm{\mathrm{nil}}$ as desired.
\item Since $\tau^{U\left(v\right)}\le\tau^{v}+1\le6\ln K$, we get:
\begin{align*}
M\left(v\right) & =M_{\Theta_{K}+1}\left(v\right)\\
 & =e^{-\frac{1}{6}}M_{\Theta_{K}}\left(U\left(v\right)\right) & \text{(from the algorithm)}\\
 & =e^{-\frac{1}{6}}M_{\Theta_{K}+1}\left(U\left(v\right)\right) & \text{(property 4 of Lemma \ref{lem:centersToComponentsHelper})}\\
 & =e^{-\frac{1}{6}}M\left(U\left(v\right)\right).
\end{align*}
\item Using property 4 of Lemma \ref{lem:centersToComponentsHelper} again,
we obtain
\[
\mathcal{C}\left(v\right)=\mathcal{C}_{\Theta_{K}}\left(U\left(v\right)\right)=\mathcal{C}_{\Theta_{K}+1}\left(U\left(v\right)\right)=\mathcal{C}\left(U\left(v\right)\right).
\]
\end{enumerate}
\end{proof}
The next lemma shows that all simple agents choose the best possible
agent as their origin neighbor.
\begin{restatable}{lem}{RESTATEneighborsSameMass}\label{lem:neighborsSameMass}
Let $U\left(v\right)$ and $M\left(v\right)$
be the results of Centers-to-Components. Then for all simple agents
$v\in V\setminus C$ such that $2\le\min_{c\in C}\mathrm{dist}_{G}\left(v,c\right)\le6\ln K-1$:
\[
U\left(v\right)=\argmax_{v'\in\mathcal{N}\left(v\right)}M\left(v'\right)
\]
\end{restatable}

\begin{proof}
Simple agents choose their origin neighbor to be the one with maximal
mass at iteration $\Theta_{K}=\left\lfloor 12\ln K\right\rfloor $.
In addition, since $\tau^{v}=\min_{c\in C}\mathrm{dist}_{G}\left(v,c\right)\le6\ln K-1$
we obtain $\tau^{U\left(v\right)}\le\tau^{v}+1\le6\ln K$, so we can
use property 4 of Lemma \ref{lem:centersToComponentsHelper} and get:
\[
M\left(U\left(v\right)\right)=M_{\Theta_{K}+1}\left(U\left(v\right)\right)=M_{\Theta_{K}}\left(U\left(v\right)\right)\ge M_{\Theta_{K}}\left(v'\right)=M_{\Theta_{K}+1}\left(v'\right)=M\left(v'\right)
\]
as desired.
\end{proof}
The following lemma shows the required properties of Centers-to-Components
under some requirements from the center set $C$.
\begin{restatable}{lem}{RESTATEcentersToComponentsCorrect}\label{lem:centersToComponentsCorrect}
Let $C\subseteq V$ be a center
set that is $2$-independent, such that every $v\in V$ holds $\min_{c\in C}\mathrm{dist}_{G}\left(v,c\right)\le6\ln K-1$.
Let $\mathcal{C}\left(v\right),U\left(v\right),M\left(v\right)$ be
the results of Centers-to-Components. For each $c\in C$, let $V_{c}$
be its corresponding component, namely, $V_{c}=\left\{ v\in V\mid\mathcal{C}\left(v\right)=c\right\} $.
Then the following properties are satisfied:
\begin{enumerate}
\item $\left\{ V_{c}\mid c\in C\right\} $ are pairwise disjoint and $V=\bigcup_{c\in C}V_{c}$.
\item $\mathcal{N}\left(c\right)\subseteq V_{c}$ and $G_{c}$ is connected
for all $c\in C$.
\item $M\left(v\right)=e^{-\frac{1}{6}d\left(v\right)}M\left(\mathcal{C}\left(v\right)\right)$
and $U\left(v\right)=\argmin_{v'\in\mathcal{N}\left(v\right)\cap V_{\mathcal{C}\left(v\right)}}d\left(v'\right)$
for all $v\in V\setminus C$.
\end{enumerate}
\end{restatable}

\begin{proof}
$ $
\begin{enumerate}
\item The components are trivially disjoint from the way we defined them.
Since for any $v\in V$ we assume $\tau^{v}=\min_{c\in C}\mathrm{dist}_{G}\left(v,c\right)\le6\ln K-1$,
we obtain from property 1 of Lemma \ref{lem:ifCloseThenLocallyGood}
that $\mathcal{C}\left(v\right)\neq\mathrm{\mathrm{nil}}$ and $v\in\bigcup_{c\in C}V_{c}$
as desired.
\item Since $C$ is 2-independent, it directly follows from the algorithm
that $\mathcal{N}\left(c\right)\subseteq V_{c}$ for all $c\in C$.
Now, let $v\in V$. For a path of connected agents $v=u_{0},\dots,u_{m}$
such that $u_{i+1}=U\left(u_{i}\right)$ for any $i<m$, we get from
property 3 of Lemma \ref{lem:ifCloseThenLocallyGood} that $\mathcal{C}\left(u_{i}\right)=\mathcal{C}\left(v\right)$
for all $i$. From property 2 of Lemma \ref{lem:ifCloseThenLocallyGood}
we also obtain that $M\left(u_{i}\right)<M\left(u_{i+1}\right)$ for
all $i<m$ such that $u_{i}\notin C$, and thus all non-center agents
on the path must be different. Hence, if $m\ge N$ we obtain that
there must be a center $u$ on the path, and since $u=\mathcal{C}\left(u\right)=\mathcal{C}\left(v\right)$,
we get that $\mathcal{C}\left(v\right)$ must be connected to $v$.
We obtain that all agents are connected to their center, and thus
$G_{c}$ is connected for all $c\in C$ as claimed.
\item We proceed by induction on $d\left(v\right)=\mathrm{dist}_{G_{\mathcal{C}\left(v\right)}}\left(v,\mathcal{C}\left(v\right)\right)$.
If $d\left(v\right)=1$ (i.e., $v$ is center-adjacent), the statement
trivially follows from the algorithm. Otherwise, we assume the statement
is true for all $v'\in V\setminus C$ such that $d\left(v'\right)<d\left(v\right)$.
Since $G_{C\left(v\right)}$ is connected from property 2, there must
be some $v'\in\mathcal{N}\left(v\right)\cap V_{\mathcal{C}\left(v\right)}$
such that $d\left(v\right)=d\left(v'\right)+1$, and thus we get from
the induction assumption that $M\left(v'\right)=e^{-\frac{1}{6}d\left(v'\right)}M\left(\mathcal{C}\left(v\right)\right)$.
From Lemma \ref{lem:neighborsSameMass}, we get that $M\left(U\left(v\right)\right)\ge M\left(v'\right)$,
and using property 2 of Lemma \ref{lem:ifCloseThenLocallyGood} we
obtain 
\begin{equation}
M\left(v\right)\ge e^{-\frac{1}{6}}M\left(v'\right)=e^{-\frac{1}{6}\left(d\left(v'\right)+1\right)}M\left(\mathcal{C}\left(v\right)\right)=e^{-\frac{1}{6}d\left(v\right)}M\left(\mathcal{C}\left(v\right)\right).\label{eq:massLargerThanExpDist}
\end{equation}
As before, from Lemma \ref{lem:ifCloseThenLocallyGood} there is a
path $v=u_{0},\dots,u_{m}=\mathcal{C}\left(v\right)$ from $v$ to
its center such that $U\left(u_{i}\right)=u_{i+1}$ for any $i<m$
and $\mathcal{C}\left(u_{i}\right)=\mathcal{C}\left(v\right)$ for
all $i$. We must have $m\ge\mathrm{dist}_{G_{\mathcal{C}\left(v\right)}}\left(v,\mathcal{C}\left(v\right)\right)=d\left(v\right)$,
and using property 2 of Lemma \ref{lem:ifCloseThenLocallyGood} iteratively
we get
\[
M\left(v\right)=e^{-\frac{1}{6}}M\left(u_{1}\right)=\cdots=e^{-\frac{1}{6}m}M\left(\mathcal{C}\left(v\right)\right)\le e^{-\frac{1}{6}d\left(v\right)}M\left(\mathcal{C}\left(v\right)\right).
\]
Combining with Eq. (\ref{eq:massLargerThanExpDist}) we get $M\left(v\right)=e^{-\frac{1}{6}d\left(v\right)}M\left(\mathcal{C}\left(v\right)\right)$
as desired. From property 3 of Lemma \ref{lem:ifCloseThenLocallyGood}
we have $U\left(v\right)\in\mathcal{N}\left(v\right)\cap V_{\mathcal{C}\left(v\right)}$,
and using Lemma \ref{lem:neighborsSameMass} we get
\begin{align*}
U\left(v\right) & =\argmax_{v'\in\mathcal{N}\left(v\right)\cap V_{\mathcal{C}\left(v\right)}}M\left(v'\right)\\
 & =\argmax_{v'\in\mathcal{N}\left(v\right)\cap V_{\mathcal{C}\left(v\right)}}e^{-\frac{1}{6}d\left(v'\right)}M\left(\mathcal{C}\left(v\right)\right)\\
 & =\argmin_{v'\in\mathcal{N}\left(v\right)\cap V_{\mathcal{C}\left(v\right)}}d\left(v'\right),
\end{align*}
concluding our proof.
\end{enumerate}
\end{proof}

\subsection{\label{subsec:analyzingInformed}Analyzing Compute-Centers-Informed}

The first thing we need to show is that the center set returned by
Compute-Centers-Informed satisfies the conditions of Lemma \ref{lem:centersToComponentsCorrect}:
\begin{restatable}{lem}{RESTATEcomputeCentersInformedCorrect}\label{lem:computeCentersInformedCorrect}
Let $C\subseteq V$ be the
center set returned by Compute-Centers-Informed. Then:
\begin{enumerate}
\item $C$ is $2$-independent.
\item For all $v\in V$, $\min_{c\in C}\mathrm{dist}_{G}\left(v,c\right)\le6\ln K-1.$
\end{enumerate}
\end{restatable}

\begin{proof}
$ $
\begin{enumerate}
\item The statement follows directly from the fact the agent $v$ that is
added to the center set at iteration $t$ holds $\min_{c\in C_{t}}\mathrm{dist}_{G}\left(v,c\right)\ge3$.
\item When the algorithm terminates there are no unsatisfied agents. Hence,
for all $v\in V$, either $\min_{c\in C}\mathrm{dist}_{G}\left(v,c\right)\le2$,
in which case we are done, or $M\left(v\right)\ge\min\left\{ \left|\mathcal{N}\left(v\right)\right|,K\right\} \ge2$.
In the latter case we obtain from properties 1 and 2 of Lemma \ref{lem:centersToComponentsHelper}:
\[
2\le M\left(v\right)=\exp\left(-\frac{1}{6}\mathrm{dist}_{G}\left(v,\mathcal{C}\left(v\right)\right)\right)M\left(\mathcal{C}\left(v\right)\right)\le\exp\left(-\frac{1}{6}\mathrm{dist}_{G}\left(v,\mathcal{C}\left(v\right)\right)\right)K,
\]
and thus $\min_{c\in C}\mathrm{dist}_{G}\left(v,c\right)\le6\ln K-1$
as desired.
\end{enumerate}
\end{proof}
Now, we can show that by using our informed graph partitioning algorithms,
the mass of all agents is large:
\begin{restatable}{thm}{RESTATEinformedPartitioningLargeMass}\label{thm:informedPartitioningLargeMass}
Let $C\subseteq V$ be the
center set returned by Compute-Centers-Informed, and let $\left\{ V_{c}\subseteq V\mid c\in C\right\} $
be the components resulted from Centers-to-Components. For every $v\in V$:
\[
M\left(v\right)\ge e^{-1}\min\left\{ \left|\mathcal{N}\left(v\right)\right|,K\right\} .
\]
\end{restatable}

\begin{proof}
For any center $v\in C$ this is trivial. Since all agents are satisfied
when the algorithm terminates, each $v\in V\setminus C$ must either
hold $M\left(v\right)\ge\min\left\{ \left|\mathcal{N}\left(v\right)\right|,K\right\} $
or $\min_{c\in C}\mathrm{dist}_{G}\left(v,c\right)\le2$. Hence, we
only need to prove the claim for each non-center agent $v\in V\setminus C$
in distance at most $2$ from the center set.

We first inspect the case that the agent is not center-adjacent, namely,
$\min_{c\in C}\mathrm{dist}_{G}\left(v,c\right)=2$. Let $t_{0}$
be the last iteration such that $\min_{c\in C_{t_{0}}}\mathrm{dist}_{G}\left(v,c\right)\ge3$.
Note that this means $\mathrm{dist}_{G}\left(v,c_{t_{0}}\right)=2$.
In the case that $v\notin S_{t_{0}}$, $v$ is satisfied, and since
$\min_{c\in C_{t_{0}}}\mathrm{dist}_{G}\left(v,c\right)\ge3$, it
must hold $M_{t_{0}}\left(v\right)\ge\min\left\{ \left|\mathcal{N}\left(v\right)\right|,K\right\} $.
Also, $c_{t_{0}}\in S_{t_{0}}$ and thus $3\le\min_{c\in C_{t_{0}}}\mathrm{dist}_{G}\left(c_{t_{0}},c\right)$
and $M_{t_{0}}\left(c_{t_{0}}\right)<\min\left\{ \left|\mathcal{N}\left(c_{t_{0}}\right)\right|,K\right\} $.
Now, property 2 of Lemma \ref{lem:centersToComponentsHelper} gives
\[
\exp\left(-\frac{1}{6}\min_{c\in C_{t_{0}}}\mathrm{dist}_{G}\left(v,c\right)\right)K\ge M_{t_{0}}\left(v\right)\ge\min\left\{ \left|\mathcal{N}\left(v\right)\right|,K\right\} .
\]
Recall that $v\in\mathcal{N}\left(v\right)$, so $\left|\mathcal{N}\left(v\right)\right|\ge2$
and thus $\exp\left(-\frac{1}{6}\min_{c\in C_{t_{0}}}\mathrm{dist}_{G}\left(v,c\right)\right)K\ge2$.
Hence, $3\le\min_{c\in C_{t_{0}}}\mathrm{dist}_{G}\left(v,c\right)\le6\ln K-6\ln2\le6\ln K-4$
and thus $3\le\min_{c\in C_{t_{0}}}\mathrm{dist}_{G}\left(c_{t_{0}},c\right)\le6\ln K-2$.
Let $u$ be an agent that is a common neighbor of $v$ and $c_{t_{0}}$,
namely, $u\in\mathcal{N}\left(v\right)\cap\mathcal{N}\left(c_{t_{0}}\right)$.
We obtain $2\le\min_{c\in C_{t_{0}}}\mathrm{dist}_{G}\left(u,c\right)\le6\ln K-3$
as well. We can now use Lemma \ref{lem:neighborsSameMass} on $c_{t_{0}}$
and $u$ to obtain
\[
\min\left\{ \left|\mathcal{N}\left(c_{t_{0}}\right)\right|,K\right\} >M_{t_{0}}\left(c_{t_{0}}\right)\ge e^{-\frac{1}{6}}M_{t_{0}}\left(u\right)\ge e^{-\frac{2}{6}}M_{t_{0}}\left(v\right)\ge e^{-\frac{2}{6}}\min\left\{ \left|\mathcal{N}\left(v\right)\right|,K\right\} .
\]
In the other case that $v\in S_{t_{0}}$, since $c_{t_{0}}=\argmax_{v'\in S_{t_{0}}}\left|\mathcal{N}\left(v'\right)\right|$,
we obtain $\left|\mathcal{N}\left(v\right)\right|\le\left|\mathcal{N}\left(c_{t_{0}}\right)\right|$,
and anyway $\min\left\{ \left|\mathcal{N}\left(c_{t_{0}}\right)\right|,K\right\} \ge e^{-\frac{2}{6}}\min\left\{ \left|\mathcal{N}\left(v\right)\right|,K\right\} $.
In all further iterations $t>t_{0}$, we can use Lemma \ref{lem:neighborsSameMass}
on $v$ to obtain
\[
M_{t}\left(v\right)\ge e^{-\frac{1}{6}}M_{t}\left(u\right)=e^{-\frac{2}{6}}\min\left\{ \left|\mathcal{N}\left(c_{t_{0}}\right)\right|,K\right\} \ge e^{-\frac{4}{6}}\min\left\{ \left|\mathcal{N}\left(v\right)\right|,K\right\} ,
\]
as desired.

Now we look at the case where $v$ is center-adjacent and $\min_{c\in C}\mathrm{dist}_{G}\left(v,c\right)=1$.
Again, let $t_{0}$ be the last iteration such that $\min_{c\in C_{t_{0}}}\mathrm{dist}_{G}\left(v,c\right)\ge2$,
and thus $\mathrm{dist}_{G}\left(v,c_{t_{0}}\right)=1$. In the case
that $v\notin S_{t_{0}}$, either $M_{t_{0}}\left(v\right)\ge\min\left\{ \left|\mathcal{N}\left(v\right)\right|,K\right\} $
or $\min_{c\in C_{t_{0}}}\mathrm{dist}_{G}\left(v,c\right)=2$, in
which case we obtain from before that $M_{t_{0}}\left(v\right)\ge e^{-\frac{4}{6}}\min\left\{ \left|\mathcal{N}\left(v\right)\right|,K\right\} $.
As before, we can use Lemma \ref{lem:neighborsSameMass} on $c_{t_{0}}$
and get
\[
\min\left\{ \left|\mathcal{N}\left(c_{t_{0}}\right)\right|,K\right\} >M_{t_{0}}\left(c_{t_{0}}\right)\ge e^{-\frac{1}{6}}M_{t_{0}}\left(v\right)\ge e^{-\frac{5}{6}}\min\left\{ \left|\mathcal{N}\left(v\right)\right|,K\right\} .
\]
In the other case that $v\in S_{t_{0}}$, again we obtain $\min\left\{ \left|\mathcal{N}\left(c_{t_{0}}\right)\right|,K\right\} \ge e^{-\frac{5}{6}}\min\left\{ \left|\mathcal{N}\left(v\right)\right|,K\right\} $.
In all further iterations $t>t_{0}$, we get
\[
M_{t}\left(v\right)=e^{-\frac{1}{6}}\min\left\{ \left|\mathcal{N}\left(c_{t_{0}}\right)\right|,K\right\} \ge e^{-1}\min\left\{ \left|\mathcal{N}\left(v\right)\right|,K\right\} ,
\]
concluding our proof.
\end{proof}
Together with Theorem \ref{thm:nonCenterRegret}, we obtain the desired
regret bound.
\begin{restatable}{cor}{RESTATEinformedRegret}\label{cor:informedRegret}
Let $T\ge K^{2}\ln K$. Let $C\subseteq V$
be the center set returned by Compute-Centers-Informed, and let $\left\{ V_{c}\subseteq V\mid c\in C\right\} $
be the components resulted from Centers-to-Components. Using the center-based
policy, we obtain for every $v\in V$:
\[
R_{T}\left(v\right)\le12\sqrt{\left(\ln K\right)\left(1+\frac{K}{\left|\mathcal{N}\left(v\right)\right|}\right)T}=\widetilde{O}\left(\sqrt{\left(1+\frac{K}{\left|\mathcal{N}\left(v\right)\right|}\right)T}\right).
\]
\end{restatable}

\subsection{\label{subsec:analyzingUninformed}Analyzing Compute-Centers-Uninformed}

First, we show that Compute-Centers-Uninformed terminates after a
relatively small number of steps, and thus the loss suffered while
running it is insignificant.
\begin{restatable}{lem}{RESTATEcomputeCentersUninformedSteps}\label{lem:computeCentersUninformedSteps}
Compute-Centers-Uninformed
runs for less than $12K\ln\left(K^{2}\bar{N}T\right)$ steps.
\end{restatable}

\begin{proof}
There are $K$ iterations in Compute-Centers-Uninformed, such that
at each iteration the agents run Luby's algorithm for $4\left\lceil 3\ln\left(\bar{N}\sqrt{KT}\right)\right\rceil $
steps, and Centers-to-Components for $\Theta_{K}+1=\left\lfloor 12\ln K\right\rfloor +1$
steps. We obtain that Compute-Centers-Uninformed terminates after
\[
\left(4\left\lceil 3\ln\left(\bar{N}\sqrt{KT}\right)\right\rceil +\left\lfloor 12\ln K\right\rfloor +1\right)K\le12K\ln\left(K^{2}\bar{N}T\right)
\]
steps.
\end{proof}
We now present a lemma that will help us with the analysis of Compute-Centers-Uninformed.
In the following, we denote $\Delta^{v}=K-\min\left\{ \left|\mathcal{N}\left(v\right)\right|,K\right\} $.
\begin{restatable}{lem}{RESTATEcomputeCenterUninformedAgentsSatisfied}\label{lem:computeCenterUninformedAgentsSatisfied}
Let $C\subseteq V$
be the center set returned by Compute-Centers-Uninformed, and let
$\left\{ V_{c}\subseteq V\mid c\in C\right\} $ be the components
resulted from Centers-to-Components. For any $v\in V$ such that $v\notin S_{\Delta^{v}}$,
either $M\left(v\right)\ge\min\left\{ \left|\mathcal{N}\left(v\right)\right|,K\right\} $,
or there is some $c\in C$ such that $\left|\mathcal{N}\left(c\right)\right|\ge e^{-\frac{1}{6}}\left|\mathcal{N}\left(v\right)\right|$
and $\mathrm{dist}_{G}\left(v,c\right)\le2$.
\end{restatable}

\begin{proof}
Let $v\in V$ be an agent such that $v\notin S_{\Delta^{v}}$. At
iteration $\Delta^{v}-1$, it follows directly from the algorithm
that either $\min_{c\in C_{\Delta^{v}-1}}\mathrm{dist}_{G}\left(v,c\right)\le2$
or $M_{\Delta^{v}-1}\left(v\right)\ge\min\left\{ \left|\mathcal{N}\left(v\right)\right|,K\right\} $.
In the first case, since $\left|\mathcal{N}\left(v\right)\right|\le\left|\mathcal{N}\left(c\right)\right|$
for all $c\in C_{\Delta^{v}-1}\subseteq C$, we are done.

Otherwise, we denote by $c^{v}=\mathcal{C}_{\Delta^{v}-1}\left(v\right)\neq\mathrm{nil}$
the center of agent $v$ at iteration $\Delta^{v}-1$. Note that from
properties 1 and 2 of Lemma \ref{lem:centersToComponentsHelper},
we obtain:
\[
e^{-\frac{1}{6}\mathrm{dist}_{G}\left(v,c^{v}\right)}M\left(c^{v}\right)\ge M_{\Delta^{v}-1}\left(v\right)\ge\min\left\{ \left|\mathcal{N}\left(v\right)\right|,K\right\} \ge2,
\]
and thus $\mathrm{dist}_{G}\left(v,c^{v}\right)\le6\ln K-1$. From
Lemma \ref{lem:ifCloseThenLocallyGood}, we get that $\mathcal{C}_{\Delta^{v}-1}\left(U_{\Delta^{v}-1}\left(v\right)\right)=c^{v}$
and
\begin{align*}
e^{-\frac{1}{6}\mathrm{dist}_{G}\left(U_{\Delta^{v}-1}\left(v\right),c^{v}\right)}M\left(c^{v}\right) & \ge M_{\Delta^{v}-1}\left(U_{\Delta^{v}-1}\left(v\right)\right)\\
 & =e^{\frac{1}{6}}M_{\Delta^{v}-1}\left(v\right)\\
 & \ge\min\left\{ \left|\mathcal{N}\left(v\right)\right|,K\right\} \\
 & \ge2.
\end{align*}
Thus, we get $\mathrm{dist}_{G}\left(U_{\Delta^{v}-1}\left(v\right),c^{v}\right)\le6\ln K-1$
as well. Using this fact iteratively, we get that there is a path
$v=u_{0},\dots,u_{m}=c^{v}$ such that $U_{\Delta^{v}-1}\left(u_{i}\right)=u_{i+1}$,
$\mathrm{dist}_{G}\left(u_{i},c^{v}\right)\le6\ln K-1$ and $M_{\Delta^{v}-1}\left(u_{i+1}\right)=e^{\frac{1}{6}}M_{\Delta^{v}-1}\left(u_{i}\right)$
for any $i<m$. Notice that this also means $M_{\Delta^{v}-1}\left(v\right)=e^{-\frac{1}{6}m}M\left(c^{v}\right)$.

Now, assume to the contrary some simple agent on the path other than
$v$ becomes a center or center-adjacent after iteration $\Delta^{v}-1$,
and let $u_{j}$ be the first such agent, where $1\le j<m-1$. Let
$u\in\mathcal{N}\left(u_{j}\right)$ be the neighbor of $u_{j}$ that
joins the center set. Note that since $\Delta^{u}\ge\Delta^{v}$,
we obtain $\left|\mathcal{N}\left(u\right)\right|\le\left|\mathcal{N}\left(v\right)\right|$.
At iteration $\Delta^{u}-1$, all agents in the path are still simple
agents (except $c^{v}$ and $u_{m-1}$), so we can use Lemma \ref{lem:neighborsSameMass}
iteratively to obtain
\begin{align*}
M_{\Delta^{u}-1}\left(u\right) & \ge e^{-\frac{1}{6}}M_{\Delta^{u}-1}\left(u_{j}\right)\\
 & \ge\dots\ge e^{-\frac{1}{6}\left(m-j\right)}M\left(u_{m-1}\right)\\
 & =e^{-\frac{1}{6}\left(m-j+1\right)}M\left(c^{v}\right)\\
 & \ge e^{-\frac{1}{6}m}M\left(c^{v}\right)\\
 & =M_{\Delta^{v}-1}\left(v\right)\\
 & \ge\min\left\{ \left|\mathcal{N}\left(v\right)\right|,K\right\} \\
 & \ge\min\left\{ \left|\mathcal{N}\left(u\right)\right|,K\right\} .
\end{align*}
Hence, $u\notin S_{\Delta^{v}}$ which gives $u\notin C_{\Delta^{u}}$,
and thus $u_{j}$ remains a simple agent. We get that all simple agents
on the path at iteration $\Delta^{v}-1$ except $v$ must remain simple
agents when the algorithm terminates. If $v$ remain a simple agent
as well, we obtain from Lemma \ref{lem:neighborsSameMass} that
\[
M\left(v\right)\ge e^{-\frac{1}{6}}M\left(u_{1}\right)\ge\dots\ge e^{-\frac{1}{6}m}M\left(c^{v}\right)=M_{\Delta^{v}-1}\left(v\right)\ge\min\left\{ \left|\mathcal{N}\left(v\right)\right|,K\right\} 
\]
as desired. We are left with the case that some $u\in\mathcal{N}\left(v\right)$
becomes a center after iteration $\Delta^{v}$, and thus $M_{\Delta^{u}-1}\left(u\right)<\min\left\{ \left|\mathcal{N}\left(u\right)\right|,K\right\} $.
We can again use Lemma \ref{lem:neighborsSameMass} iteratively to
get
\begin{align*}
\min\left\{ \left|\mathcal{N}\left(u\right)\right|,K\right\}  & >M_{\Delta^{u}-1}\left(u\right)\\
 & \ge e^{-\frac{1}{6}}M_{\Delta^{u}-1}\left(v\right)\\
 & \ge\dots\ge e^{-\frac{1}{6}m}M\left(u_{m-1}\right)\\
 & =e^{-\frac{1}{6}\left(m+1\right)}M\left(c^{v}\right)\\
 & =e^{-\frac{1}{6}}M_{\Delta^{v}-1}\left(v\right)\\
 & \ge e^{-\frac{1}{6}}\min\left\{ \left|\mathcal{N}\left(v\right)\right|,K\right\} ,
\end{align*}
concluding our proof.
\end{proof}
As in the informed setting, we need to show the center set resulted
from Compute-Centers-Uninformed satisfies the conditions of Lemma
\ref{lem:centersToComponentsCorrect}.
\begin{restatable}{lem}{RESTATEcomputeCentersUninformedCorrect}\label{lem:computeCentersUninformedCorrect}
Let $C\subseteq V$ be
the center set resulted from Compute-Centers-Uninformed, such that
Luby's algorithm succeeded at all iterations of the algorithm. Then:
\begin{enumerate}
\item $C$ is $2$-independent.
\item For all $v\in V$, $\min_{c\in C}\mathrm{dist}_{G}\left(v,c\right)\le6\ln K-1$.
\end{enumerate}
\end{restatable}

\begin{proof}
$ $
\begin{enumerate}
\item We get that at each iteration, a 2-independent set is added to the
center set, such that every agent $v$ in that set holds $\min_{c\in C_{t-1}}\mathrm{dist}_{G}\left(v,c\right)\ge3$.
Hence, the final center set is 2-independent as claimed.
\item From Lemma \ref{lem:computeCenterUninformedAgentsSatisfied} we have
either $\min_{c\in C}\mathrm{dist}_{G}\left(v,c\right)\le2$, in which
case we are done, or $M\left(v\right)\ge\min\left\{ \left|\mathcal{N}\left(v\right)\right|,K\right\} \ge2$.
In the latter case we obtain:
\begin{align*}
2 & \le M\left(v\right)\\
 & \le\exp\left(-\frac{1}{6}\mathrm{dist}_{G}\left(v,\mathcal{C}\left(v\right)\right)\right)M\left(\mathcal{C}\left(v\right)\right) & \textrm{(Properties 1 and 2 of Lemma \ref{lem:centersToComponentsHelper})}\\
 & \le\exp\left(-\frac{1}{6}\mathrm{dist}_{G}\left(v,\mathcal{C}\left(v\right)\right)\right)K,
\end{align*}
and thus $\min_{c\in C}\mathrm{dist}_{G}\left(v,c\right)\le6\ln K-1$
as desired.
\end{enumerate}
\end{proof}
We can now obtain the same result as in the informed setting:
\begin{restatable}{thm}{RESTATEuninformedPartitioningLargeMass}\label{thm:uninformedPartitioningLargeMass}
Let $C\subseteq V$ be
the center set resulted from Compute-Centers-Uninformed, such that
Luby's algorithm succeeded at all iterations of the algorithm, and
also let $\left\{ V_{c}\subseteq V\mid c\in C\right\} $ be the components
resulted from Centers-to-Components. For every $v\in V$:
\[
M\left(v\right)\ge e^{-1}\min\left\{ \left|\mathcal{N}\left(v\right)\right|,K\right\} .
\]
\end{restatable}

\begin{proof}
In the case that $v\in S_{\Delta^{v}}$, since $W_{\Delta^{v}}$ is
a maximal 2-independet set of $S_{\Delta^{v}}$, we get that either
$v\in W_{\Delta^{v}}\subseteq C$ or $\mathrm{dist}_{G}\left(v,v'\right)\le2$
for some $v'\in W_{\Delta^{v}}\subseteq C$. In the case that $v\notin S_{\Delta^{v}}$,
we obtain from Lemma \ref{lem:computeCenterUninformedAgentsSatisfied}
that either $M\left(v\right)\ge\min\left\{ \left|\mathcal{N}\left(v\right)\right|,K\right\} $,
in which case we are done, or there is some center $c'\in C$ such
that $\mathrm{dist}_{G}\left(v,c'\right)\le2$ and $e^{-\frac{1}{6}}\min\left\{ \left|\mathcal{N}\left(v\right)\right|,K\right\} \le\min\left\{ \left|\mathcal{N}\left(c'\right)\right|,K\right\} $.

Hence we only need to prove the theorem for the case that there is
some center $c'\in C$ such that $\mathrm{dist}_{G}\left(v,c'\right)\le2$
and $e^{-\frac{1}{6}}\min\left\{ \left|\mathcal{N}\left(v\right)\right|,K\right\} \le\min\left\{ \left|\mathcal{N}\left(c'\right)\right|,K\right\} $.
We first inspect the case that $v$ is not a center or center-adjacent.
Let $u$ be an agent that is a common neighbor of $v$ and $c'$,
namely, $u\in\mathcal{N}\left(v\right)\cap\mathcal{N}\left(c'\right)$.
Lemma \ref{lem:neighborsSameMass} yields
\[
M\left(v\right)\ge e^{-\frac{1}{6}}M\left(u\right)=e^{-\frac{2}{6}}\min\left\{ \left|\mathcal{N}\left(c'\right)\right|,K\right\} \ge e^{-\frac{3}{6}}\min\left\{ \left|\mathcal{N}\left(v\right)\right|,K\right\} ,
\]
as desired. If $v$ is a center the claim is trivial, so we are left
with the case that $v$ is center-adjacent to a center $c\in C$.
Note that $\mathrm{dist}_{G}\left(c,c'\right)\le3$. If $\min\left\{ \left|\mathcal{N}\left(c'\right)\right|,K\right\} \le\min\left\{ \left|\mathcal{N}\left(c\right)\right|,K\right\} $
we are done. Otherwise, in the case that $\min\left\{ \left|\mathcal{N}\left(c\right)\right|,K\right\} <\min\left\{ \left|\mathcal{N}\left(c'\right)\right|,K\right\} $,
we obtain
\begin{align*}
\min\left\{ \left|\mathcal{N}\left(c\right)\right|,K\right\}  & >M_{\Delta^{c}-1}\left(c\right)\\
 & \ge e^{-\frac{3}{6}}M_{\Delta^{c}-1}\left(c'\right) & \text{(iterative application of Lemma \ref{lem:neighborsSameMass})}\\
 & =e^{-\frac{3}{6}}\min\left\{ \left|\mathcal{N}\left(c'\right)\right|,K\right\}  & (c'\in C_{\Delta^{c}-1})\\
 & \ge e^{-\frac{4}{6}}\min\left\{ \left|\mathcal{N}\left(v\right)\right|,K\right\} .
\end{align*}
Hence,
\[
M\left(v\right)=e^{-\frac{1}{6}}\min\left\{ \left|\mathcal{N}\left(c\right)\right|,K\right\} \ge e^{-\frac{5}{6}}\min\left\{ \left|\mathcal{N}\left(v\right)\right|,K\right\} ,
\]
concluding our proof.
\end{proof}
Again we can use Theorem \ref{thm:nonCenterRegret} to obtain the
desired regret bound.
\begin{restatable}{cor}{RESTATEuninformedRegret}\label{cor:uninformedRegret}
 Let $T\ge K^{2}\ln K$ and $\bar{N}\ge N$.
Let $C\subseteq V$ be the center set resulted from Compute-Centers-Uninformed,
and let $\left\{ V_{c}\subseteq V\mid c\in C\right\} $ be the components
resulted from Centers-to-Components. Using the center-based policy,
we obtain for every $v\in V$:
\[
R_{T}\left(v\right)\le12\left(K\ln\left(K^{2}\bar{N}T\right)+\sqrt{\left(\ln K\right)\left(1+\frac{K}{\left|\mathcal{N}\left(v\right)\right|}\right)T}\right)+1=\widetilde{O}\left(\sqrt{\left(1+\frac{K}{\left|\mathcal{N}\left(v\right)\right|}\right)T}\right).
\]
\end{restatable}

\begin{proof}
Luby's algorithm succeeds with probability $1-\frac{1}{KT}$ at each
iteration of Compute-Centers-Uninformed. Hence, from the union bound,
it succeeds at all iterations with probability $1-\frac{1}{T}$. In
that case, from Lemma \ref{lem:computeCentersUninformedCorrect},
we can use Theorem \ref{thm:nonCenterRegret} and Theorem \ref{thm:uninformedPartitioningLargeMass}
to bound the expected regret of agent $v$ after Compute-Centers-Uninformed
finished by:
\[
7\sqrt{\left(\ln K\right)\frac{K}{M\left(v\right)}T}\le7\sqrt{\left(\ln K\right)e\frac{K}{\min\left\{ \left|\mathcal{N}\left(v\right)\right|,K\right\} }T}\le12\sqrt{\left(\ln K\right)\left(1+\frac{K}{\left|\mathcal{N}\left(v\right)\right|}\right)T}.
\]
From Lemma \ref{lem:computeCentersUninformedSteps}, Compute-Centers-Uninformed
finishes after no more than $12K\ln\left(K^{2}\bar{N}T\right)$ steps,
so the overall expected regret in this case is bounded by:
\[
12\left(K\ln\left(K^{2}\bar{N}T\right)+\sqrt{\left(\ln K\right)\left(1+\frac{K}{\left|\mathcal{N}\left(v\right)\right|}\right)T}\right).
\]
In the case that Luby's algorithm failed at one of the iterations,
we can bound the regret by $T$, the maximal regret possible. Hence,
we obtain the desired result:
\begin{align*}
R_{T}\left(v\right) & \le12\left(1-\frac{1}{T}\right)\left(K\ln\left(K^{2}\bar{N}T\right)+\sqrt{\left(\ln K\right)\left(1+\frac{K}{\left|\mathcal{N}\left(v\right)\right|}\right)T}\right)+\frac{1}{T}T\\
 & \le12\left(K\ln\left(K^{2}\bar{N}T\right)+\sqrt{\left(\ln K\right)\left(1+\frac{K}{\left|\mathcal{N}\left(v\right)\right|}\right)T}\right)+1.
\end{align*}
\end{proof}

\subsection{\label{subsec:analyzingAverageRegret}Average regret of the center-based
policy}

As mentioned before, we strictly improve the result of \citet{cesa2019delay},
and our algorithms imply the same average expected regret bound.
\begin{restatable}{cor}{RESTATEaverageRegret}\label{cor:averageRegret}
Let $T\ge K^{2}\ln K$. Let $C\subseteq V$
be the center set resulted from Compute-Centers-Informed or Compute-Centers-Uninformed,
and let $\left\{ V_{c}\subseteq V\mid c\in C\right\} $ be the components
resulted from Centers-to-Components. Using the center-based policy,
we get:
\[
\frac{1}{N}\sum_{v\in V}R_{T}\left(v\right)=\widetilde{O}\left(\sqrt{\left(1+\frac{K}{N}\alpha\left(G\right)\right)T}\right).
\]
\end{restatable}

\begin{proof}
Using either Compute-Centers-Informed or Compute-Centers-Uninformed
to partition the graph for the center-based policy, we get from Corollaries
\ref{cor:informedRegret} and \ref{cor:uninformedRegret} that for
all $v\in V$:
\[
R_{T}\left(v\right)=\widetilde{O}\left(\sqrt{\left(1+\frac{K}{\left|\mathcal{N}\left(v\right)\right|}\right)T}\right).
\]
Hence,
\[
\frac{1}{N}\sum_{v\in V}R_{T}\left(v\right)=\widetilde{O}\left(\frac{1}{N}\sum_{v\in V}\sqrt{\left(1+\frac{K}{\left|\mathcal{N}\left(v\right)\right|}\right)T}\right)=\widetilde{O}\left(\frac{1}{\sqrt{N}}\sqrt{\sum_{v\in V}\left(1+\frac{K}{\left|\mathcal{N}\left(v\right)\right|}\right)T}\right),
\]
 where the last equality is due to the Cauchy\textendash Schwarz inequality.
Since $\sum_{v\in V}\frac{1}{\left|\mathcal{N}\left(v\right)\right|}\le\alpha\left(G\right)$
\citep{wei1981lower}, we obtain:
\[
\frac{1}{N}\sum_{v\in V}R_{T}\left(v\right)=\widetilde{O}\left(\sqrt{\left(1+\frac{K}{N}\sum_{v\in V}\frac{1}{\left|\mathcal{N}\left(v\right)\right|}\right)T}\right)=\widetilde{O}\left(\sqrt{\left(1+\frac{K}{N}\alpha\left(G\right)\right)T}\right)
\]
as desired.
\end{proof}

\section{\label{sec:conclusions}Conclusions}

We investigated the cooperative nonstochastic multi-armed bandit problem,
and presented the center-based cooperation policy (Algorithms \ref{alg:center-based-policy-centers}
and \ref{alg:center-based-policy-non-centers}). We provided partitioning
algorithms that provably yield a low individual regret bound that
holds simultaneously for all agents (Algorithms \ref{alg:centersToComponents},
\ref{alg:comupteCentersInformed} and \ref{alg:computeCentersUninformed}).
We express this bound in terms of the agents' degree in the communication
graph. This bound strictly improves a previous regret bound from \citep{cesa2019delay}
(Corollary \ref{cor:averageRegret}), and also resolves an open question
from that paper.

Note that our regret bound in the informed setting does not depend
on the total number of agents, $N$, and in the uninformed setting
it depends on $\bar{N}$ only logarithmically. It is unclear whether
in the uninformed setting, any dependence on $N$ in the individual
regret is required.

\section*{\label{sec:acknowledgments}Acknowledgments}

This work was supported in part by the Yandex Initiative in Machine
Learning and by a grant from the Israel Science Foundation (ISF).

\bibliographystyle{plainnat}
\bibliography{individual-regret-in-cooperative-nonstochastic-multi-armed-bandits}

\clearpage{}

\appendix

\section{\label{sec:lubysAlgorithm}Luby's algorithm}

Let $G=\left\langle V,E\right\rangle $ be an undirected connected
graph and let $U\subseteq V$. We can find a 2-MIS of $U$ in a distributed
manner with high probability by using Luby's algorithm \citep{luby1986simple,alon1986fast}
on $\left(G^{2}\right)_{|U}$, detailed in Algorithm \ref{alg:lubys}.

At each iteration of the algorithm, every agent in $U$ picks a number
uniformly from $\left[0,1\right]$. Agents that picked the maximal
number among their neighbors of distance 2 join the 2-MIS, and their
neighbors of distance 2 stop participating. A 2-MIS is computed after
$T_{\delta}=\left\lceil 3\ln\left(\frac{\text{\ensuremath{\left|V\right|}}}{\sqrt{\delta}}\right)\right\rceil $
iterations with probability $1-\delta$.

To simulate communication over $G^{2}$, we use 2 steps to deliver
a message. First, the agents send their message. Then, the agents
send a message based on the messages they received in the previous
step. In Luby's algorithm, agents only need to know the agent in their
neighborhood with the maximal random number, or whether an agent in
their neighborhood joined the MIS. Hence, every message has length
of order $\widetilde{O}\left(1\right)$.
\begin{algorithm}
\caption{\label{alg:lubys}Luby's algorithm on $\left(G^{2}\right)_{|U}$ -
agent $v$}
\begin{algorithmic}[1]

\Parameters{Agent set $U\subseteq V$; Error probability $\delta>0$.}

\Init{Participating agents $P_{0}=U$.}

\State{$T_{\delta}=\left\lceil 3\ln\left(\frac{\text{\ensuremath{\left|V\right|}}}{\sqrt{\delta}}\right)\right\rceil $}

\For{$1\le t\le T_{\delta}$}

\If{$v\in P_{t}$}

\State{Pick a number $r_{t}^{v}$ uniformly from $\left[0,1\right]$.}

\State{Send the following message to the set $\mathcal{N}\left(v\right)$:
$m_{t,1}\left(v\right)=\left\langle v,t,1,r_{t}^{v}\right\rangle $.}

\EndIf

\State{Receive all messages $m_{t,1}\left(v'\right)$ from $v'\in\mathcal{N}\left(v\right)$.}

\If{$\mathcal{N}\left(v\right)\cap P_{t}\neq\emptyset$}

\State{Set $u_{t}=\argmax_{v'\in\mathcal{N}\left(v\right)\cap P_{t}}\left(r_{t}^{v'}\right)$.}

\State{Send the following message to the set $\mathcal{N}\left(v\right)$:
$m_{t,2}\left(v\right)=\left\langle u_{t},t,2,r_{t}^{u}\right\rangle $.}

\EndIf

\State{Receive all messages $m_{t,2}\left(v'\right)$ from $v'\in\mathcal{N}\left(v\right)$.}

\If{$v=\argmax_{v'\in P_{t}\land\mathrm{dist}_{G}\left(v,v'\right)\le2}\left(r_{t}^{v'}\right)$}

\State{\textbf{Join the 2-MIS of $U$.}}

\State{Send the following message to the set $\mathcal{N}\left(v\right)$:
$m_{t,3}\left(v\right)=\left\langle v,t,3,\textrm{JOINED}\right\rangle $.}

\EndIf

\State{Receive all messages $m_{t,3}\left(v'\right)$ from $v'\in\mathcal{N}\left(v\right)$.}

\If{$\exists v'\in\mathcal{N}\left(v\right)\left(\textrm{\ensuremath{v'} joined the 2-MIS}\right)$}

\State{Send the following message to the set $\mathcal{N}\left(v\right)$:
$m_{t,4}\left(v\right)=\left\langle v,t,4,\textrm{NEIGHBOR-JOINED}\right\rangle $.}

\EndIf

\State{Receive all messages $m_{t,4}\left(v'\right)$ from $v'\in\mathcal{N}\left(v\right)$.}

\If{$v\in P_{t}$ and $\exists v'\in P_{t}\left(\mathrm{dist}_{G}\left(v,v'\right)\le2\land\textrm{\ensuremath{v'} joined the 2-MIS}\right)$}

\State{Stop participating: $v\notin P_{t+1}$.}

\ElsIf{$v\in P_{t}$ }

\State{Continue participating: $v\in P_{t+1}$.}

\EndIf

\EndFor

\end{algorithmic}
\end{algorithm}

For completeness we also provide an overview of the analysis. It follows
directly from the algorithm that it outputs an independent set of
$\left(G^{2}\right)_{|U}$. We only need to show it is maximal with
high probability, and we prove it using the following lemma (for proof,
see \citep{luby1986simple,alon1986fast}):
\begin{restatable}{lem}{RESTATElubysEdgesDecrease}\label{lem:lubysEdgesDecrease}
Let $P_{t}\subseteq U$ be the set
of participating agents at iteration $t$ of Luby's algorithm on $\left(G^{2}\right)_{|U}$,
and let $m_{t}$ be the number of edges of $\left(G^{2}\right)_{|P_{t}}$.
We obtain for all $t\ge1$:
\[
\mathbb{E}\left[m_{t+1}\right]\le\frac{1}{2}\mathbb{E}\left[m_{t}\right].
\]
\end{restatable}

With this lemma we can now show Luby's algorithm indeed outputs a
2-MIS with high probability.
\begin{restatable}{cor}{RESTATElubyCorrect}\label{cor:lubyCorrect}
Let $W\subseteq U$ be the result of Luby's
algorithm on $\left(G^{2}\right)_{|U}$. Then with probability $1-\frac{1}{\delta}$,
$W$ is a 2-MIS of $U$.
\end{restatable}

\begin{proof}
As we previously mentioned, we only need to show $W$ is a maximal
independent set with probability $1-\delta$. This is equivalent to
the statement that $P_{T_{\delta}+1}$ is empty. If we denote the
number of edges of $\left(G^{2}\right)_{|P_{t}}$ by $m_{t}$ , we
get that it suffices to prove that $m_{T_{\delta}}=0$ with high probability.
By an iterative application of Lemma \ref{lem:lubysEdgesDecrease}
we obtain:
\[
\mathbb{E}\left[m_{T_{\delta}}\right]\le\frac{1}{2}\mathbb{E}\left[m_{T_{\delta}-1}\right]\le\dots\le\frac{1}{2^{T_{\delta}}}\mathbb{E}\left[m_{0}\right]\le\frac{\left|V\right|^{2}}{2^{T_{\delta}}}.
\]
Hence, we can conclude our proof with Markov's inequality:
\[
\Pr\left[m_{T_{\delta}}\neq0\right]=\Pr\left[m_{T_{\delta}}\ge1\right]\le\mathbb{E}\left[m_{T_{\delta}}\right]\le\frac{\left|V\right|^{2}}{2^{T_{\delta}}}=\frac{\left|V\right|^{2}}{2^{\left\lceil 3\ln\left(\frac{\text{\ensuremath{\left|V\right|}}}{\sqrt{\delta}}\right)\right\rceil }}\le\delta.
\]
\end{proof}

\section{\label{sec:preliminariesProofs}Supplementary material to Section
\ref{sec:preliminaries}}

The exponential-weights algorithm is given in Algorithm \ref{alg:exp3}.
For completeness, we also give proofs for the preliminary lemmas.
\begin{algorithm}
\caption{\label{alg:exp3}The exponential-weights algorithm (Exp3)}
\begin{algorithmic}[1]

\Parameters{Number of arms $K$; Time horizon $T$; Learning rate
$\eta\left(v\right)$.}

\Init{$w_{1}^{v}\left(i\right)\leftarrow\frac{1}{K}$ for all $i\in A$.}

\For{$1\le t\le T$}

\State{Set $p_{t}^{v}\left(i\right)\leftarrow\frac{w_{t}^{v}\left(i\right)}{W_{t}^{v}}$
for all $i\in A$, where $W_{t}^{v}=\sum_{i\in A}w_{t}^{v}\left(i\right)$.}

\State{Play an action $I_{t}\left(v\right)$ drawn from $\boldsymbol{p}_{t}^{v}=\left\langle p_{t}^{v}\left(1\right),\dots,p_{t}^{v}\left(K\right)\right\rangle $.}

\State{Observe loss $\ell_{t}\left(I_{t}\left(v\right)\right)$.}

\State{Update for all $i\in A$: $w_{t+1}^{v}\left(i\right)\leftarrow w_{t}^{v}\left(i\right)\exp\left(-\eta\left(v\right)\hat{\ell}_{t}^{v}\left(i\right)\right)$,
where
\[
\hat{\ell}_{t}^{v}\left(i\right)=\frac{\ell_{t}\left(i\right)}{\mathbb{E}_{t}\left[B_{t}^{v}\left(i\right)\right]}B_{t}^{v}\left(i\right),
\]
and $B_{t}^{v}\left(i\right)$ it the event that $v$ observed $\ell_{t}\left(I_{t}\left(v\right)\right)$.}

\EndFor

\end{algorithmic}
\end{algorithm}

\subsubsection*{Proof of Lemma \ref{lem:expCoopBoundRegret}}

\RESTATEexpCoopBoundRegret*
\begin{proof}
We have
\begin{align*}
\frac{W_{t+1}^{v}}{W_{t}^{v}} & =\sum_{i\in A}\frac{w_{t+1}^{v}\left(i\right)}{W_{t}^{v}}\\
 & =\sum_{i\in A}\frac{w_{t}^{v}\left(i\right)}{W_{t}^{v}}\exp\left(-\eta\left(v\right)\hat{\ell}_{t}^{v}\left(i\right)\right)\\
 & =\sum_{i\in A}p_{t}^{v}\left(i\right)\exp\left(-\eta\left(v\right)\hat{\ell}_{t}^{v}\left(i\right)\right)\\
 & \le\sum_{i\in A}p_{t}^{v}\left(i\right)\left(1-\eta\left(v\right)\hat{\ell}_{t}^{v}\left(i\right)+\eta\left(v\right)^{2}\hat{\ell}_{t}^{v}\left(i\right)^{2}\right) & (e^{-x}\le1-x+\frac{1}{2}x^{2}\text{ for }x\ge0)\\
 & =1-\eta\left(v\right)\sum_{i\in A}p_{t}^{v}\left(i\right)\hat{\ell}_{t}^{v}\left(i\right)+\frac{\eta\left(v\right)^{2}}{2}\sum_{i\in A}p_{t}^{v}\left(i\right)\hat{\ell}_{t}^{v}\left(i\right)^{2}.
\end{align*}
Taking logs and using $\ln\left(1+x\right)\le x$ we obtain
\[
\ln\frac{W_{t+1}^{v}}{W_{t}^{v}}\le-\eta\left(v\right)\sum_{i\in A}p_{t}^{v}\left(i\right)\hat{\ell}_{t}^{v}\left(i\right)+\frac{\eta\left(v\right)^{2}}{2}\sum_{i\in A}p_{t}^{v}\left(i\right)\hat{\ell}_{t}^{v}\left(i\right)^{2}.
\]
Summing gives
\begin{equation}
\ln W_{T+1}^{v}\le-\eta\left(v\right)\sum_{t=1}^{T}\sum_{i\in A}p_{t}^{v}\left(i\right)\hat{\ell}_{t}^{v}\left(i\right)+\frac{\eta\left(v\right)^{2}}{2}\sum_{t=1}^{T}\sum_{i\in A}p_{t}^{v}\left(i\right)\hat{\ell}_{t}^{v}\left(i\right)^{2}.\label{eq:lnWLessThan}
\end{equation}
Now, for any fixed action $k$ we also have
\[
\ln W_{T+1}^{v}\ge\ln w_{T+1}^{v}\left(k\right)=-\eta\left(v\right)\sum_{t=1}^{T}\hat{\ell}_{t}^{v}\left(k\right)-\ln K.
\]
Combining with Eq. (\ref{eq:lnWLessThan}) we obtain
\[
\sum_{t=1}^{T}\sum_{i\in A}p_{t}^{v}\left(i\right)\hat{\ell}_{t}^{v}\left(i\right)-\sum_{t=1}^{T}\hat{\ell}_{t}^{v}\left(k\right)\le\frac{\ln K}{\eta\left(v\right)}+\frac{\eta\left(v\right)^{2}}{2}\sum_{t=1}^{T}\sum_{i\in A}p_{t}^{v}\left(i\right)\hat{\ell}_{t}^{v}\left(i\right)^{2}.
\]
This is true for every $k\in A$. Note that $\mathbb{E}\left[\cdot\right]=\mathbb{E}\left[\mathbb{E}_{t}\left[\cdot\right]\right]$,
and since $\mathbb{E}_{t}\left[\ell_{t}\left(I_{t}\left(v\right)\right)\right]=\sum_{i\in A}p_{t}^{v}\left(i\right)\ell_{t}\left(i\right)$
and $\mathbb{E}_{t}\left[\hat{\ell}_{t}^{v}\left(i\right)\right]=\ell_{t}\left(i\right)$,
we get
\begin{align*}
R_{T}\left(v\right) & =\mathbb{E}\left[\sum_{t=1}^{T}\ell_{t}\left(I_{t}\left(v\right)\right)-\min_{i\in A}\sum_{t=1}^{T}\ell_{t}\left(i\right)\right]\\
 & \le\mathbb{E}\left[\sum_{t=1}^{T}\ell_{t}\left(I_{t}\left(v\right)\right)\right]-\min_{i\in A}\mathbb{E}\left[\sum_{t=1}^{T}\ell_{t}\left(i\right)\right]\\
 & =\mathbb{E}\left[\sum_{t=1}^{T}\sum_{i\in A}p_{t}^{v}\left(i\right)\hat{\ell}_{t}^{v}\left(i\right)\right]-\min_{i\in A}\mathbb{E}\left[\sum_{t=1}^{T}\hat{\ell}_{t}^{v}\left(i\right)\right]\\
 & \le\frac{\ln K}{\eta\left(v\right)}+\frac{\eta\left(v\right)^{2}}{2}\mathbb{E}\left[\sum_{t=1}^{T}\sum_{i\in A}p_{t}^{v}\left(i\right)\hat{\ell}_{t}^{v}\left(i\right)^{2}\right]
\end{align*}
as desired.
\end{proof}

\subsubsection*{Proof of Lemma \ref{lem:expCoopBoundProbabilityDifference}}

\RESTATEexpCoopBoundProbabilityDifference*
\begin{proof}
From the exponential-weights update rule we have
\begin{align*}
p_{t+1}^{v}\left(i\right) & =\frac{w_{t+1}^{v}\left(i\right)}{W_{t+1}^{v}}\\
 & =\frac{W_{t}^{v}}{W_{t+1}^{v}}\exp\left(-\eta\left(v\right)\hat{\ell}_{t}^{v}\left(i\right)\right)p_{t}^{v}\left(i\right)\\
 & \ge\exp\left(-\eta\left(v\right)\hat{\ell}_{t}^{v}\left(i\right)\right)p_{t}^{v}\left(i\right) & (W_{t+1}^{v}\le W_{t}^{v})\\
 & \ge\left(1-\eta\left(v\right)\hat{\ell}_{t}^{v}\left(i\right)\right)p_{t}^{v}\left(i\right). & (1-x\le e^{-x})
\end{align*}
as stated in the first inequality in the lemma. For the second inequality,
note that
\begin{equation}
p_{t}^{v}\left(i\right)\hat{\ell}_{t}^{v}\left(i\right)=p_{t}^{v}\left(i\right)\frac{\ell_{t}\left(i\right)}{\mathbb{E}_{t}\left[B_{t}^{v}\left(i\right)\right]}B_{t}^{v}\left(i\right)\le\frac{p_{t}^{v}\left(i\right)}{\mathbb{E}_{t}\left[B_{t}^{v}\left(i\right)\right]}\le1.\label{eq:plSmallOne}
\end{equation}
Hence,
\begin{align*}
p_{t+1}^{v}\left(i\right) & =\frac{w_{t+1}^{v}\left(i\right)}{W_{t+1}^{v}}\\
 & \le\frac{w_{t}^{v}\left(i\right)}{W_{t+1}^{v}}\\
 & =\frac{\sum_{j\in A}w_{t}^{v}\left(j\right)}{\sum_{j\in A}w_{t}^{v}\left(j\right)\exp\left(-\eta\left(v\right)\hat{\ell}_{t}^{v}\left(j\right)\right)}p_{t}^{v}\left(i\right)\\
 & \le\frac{\sum_{j\in A}w_{t}^{v}\left(j\right)}{\sum_{j\in A}w_{t}^{v}\left(j\right)\left(1-\eta\left(v\right)\hat{\ell}_{t}^{v}\left(j\right)\right)}p_{t}^{v}\left(i\right) & (1-x\le e^{-x})\\
 & =\frac{1}{1-\eta\left(v\right)\sum_{j\in A}p_{t}^{v}\left(j\right)\hat{\ell}_{t}^{v}\left(j\right)}p_{t}^{v}\left(i\right)\\
 & \le\frac{1}{1-\eta\left(v\right)K}p_{t}^{v}\left(i\right). & \text{(Eq. \eqref{eq:plSmallOne})}
\end{align*}
Assuming $\eta\left(v\right)\le\frac{1}{2K}$, we obtain the desired
bound.
\end{proof}

\end{document}